\documentclass{article}
\PassOptionsToPackage{numbers, compress}{natbib}
\usepackage[main, final]{neurips_2026}
\usepackage[utf8]{inputenc} % allow utf-8 input
\usepackage[T1]{fontenc}    % use 8-bit T1 fonts
\usepackage{hyperref}       % hyperlinks
\usepackage{url}            % simple URL typesetting
\usepackage{booktabs}       % professional-quality tables
\usepackage{amsfonts}       % blackboard math symbols
\usepackage{nicefrac}       % compact symbols for 1/2, etc.
\usepackage{microtype}      % microtypography
\usepackage{xcolor}         % colors
\usepackage{amsmath}
\usepackage{amssymb}
\usepackage{comment}
\usepackage{algorithm}
\usepackage{algorithmicx}
\usepackage{algpseudocode}
\usepackage{amsthm}
\usepackage{mathrsfs}
\usepackage{booktabs}
\usepackage{multirow}
\usepackage{graphicx}
\usepackage{caption}
\usepackage{tikz}
\usepackage{placeins}

\newcommand{\M}{\mathcal{M}}
\newcommand{\E}{\mathbb{E}}

\title{UniVer: A Unified Perspective for \\Multi-step and Multi-draft Speculative Decoding}

\author{
Yepeng Weng$^{1,2}$\thanks{Equal contribution.} \quad Qiao Hu$^{3}$\footnotemark[1]\hspace{0.4em}\thanks{Corresponding author.} \quad Takehisa Yairi$^{1}$\\
$^1$ The University of Tokyo \quad $^2$Lenovo AI Technology Center \\ $^3$ National Center for Mathematics and Interdisciplinary Sciences (NCMIS), AMSS, CAS \\
\texttt{yweng@g.ecc.u-tokyo.ac.jp, huqiao2020@amss.ac.cn}}

\begin{document}

\maketitle

\begin{abstract}
Speculative decoding accelerates Large Language Models via draft-then-verify, where verification can be framed as an Optimal Transport (OT) problem. Existing approaches typically handle multi-draft and multi-step aspects in isolation, applying either flat OT to single-step drafts or per-token rejection sampling to tree-structured candidates. This separation leaves the joint regime (where multi-step dependencies meet multi-draft branching) poorly optimized, as local verification rules fail to exploit the coupling between horizontal and vertical dimensions of candidate trees. In this paper, we propose a unified perspective that casts tree-based verification as a conditional OT problem. Our key insight is that vertical dependencies can be abstracted through prefix acceptance probabilities, which act as dynamic scaling factors to actively guide horizontal draft selection. Based on this principle, we introduce UniVer, a verification algorithm that jointly optimizes across tree levels by composing local optimal transport plans under prefix constraints. We prove that UniVer remains lossless and achieves the optimal acceptance rate under the proposed conditional framework. Extensive experiments across different tasks and models demonstrate that UniVer improves acceptance length by 4.2\% to 8.5\% over standard recursive rejection sampling without replacement and by up to 2.5\% over the single-step OT-based Greedy method, while preserving the target distribution exactly.
\end{abstract}

\section{Introduction}\label{sec:intro}

Modern Large Language Models (LLMs) suffer from high inference latency due to their autoregressive nature. Speculative decoding~\cite{leviathan2023fast, chen2023accelerating} mitigates this problem without compromising output quality: a lightweight draft model predicts future tokens, which are then verified by the target model in parallel. A fundamental algorithmic challenge is to design an optimal verification strategy that maximizes the expected acceptance rate while maintaining statistical fidelity to the original model (\emph{i.e.,} the output distribution is identical to that of vanilla autoregressive decoding).

Recent studies improve acceptance rates by extending the naive (local) verification strategy along two orthogonal dimensions: optimizing horizontal selection among multiple candidates and managing vertical dependencies across multiple steps.
On one hand, multi-draft methods such as SpecTr~\cite{sun2023spectr} and Greedy method~\cite{hu2025towards} select multiple draft tokens at each step and optimize the acceptance rates from an Optimal Transport (OT) perspective. However, these approaches remain vertically "myopic", treating each generation step as an independent event and failing to account for the sequential dependencies in tree-structured drafts.
On the other hand, multi-step methods~\cite{sun2025block, weng2025traversal} focus on optimizing verification vertically across sequences or trees. Yet, these methods are horizontally limited: they rely on local, per-token rejection sampling, which is suboptimal compared to coordinated, multi-draft selection.
Consequently, existing paradigms suffer from a structural trade-off: they either optimize across multiple candidates at a single depth or across multiple steps for a single candidate, but fail to achieve joint optimality across both dimensions (Figure \ref{fig:overview}).

\begin{figure*}[t]
\centering
\includegraphics[width=0.93\textwidth]{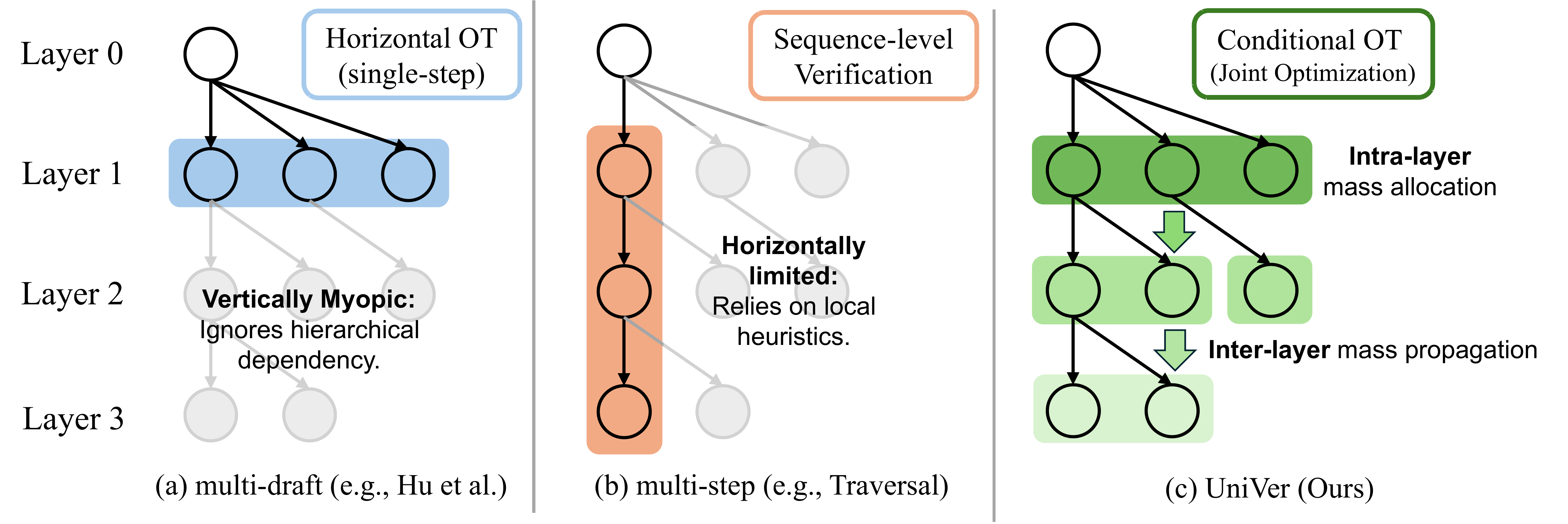}
\caption{Conceptual illustration of verification paradigms. \textbf{(a) Multi-Draft approaches} (e.g., \citet{hu2025towards}) optimize horizontal selection via OT but ignore vertical dependencies (\emph{vertically myopic}). \textbf{(b) Multi-Step approaches} (e.g., Traversal Verification~\cite{weng2025traversal}) optimize vertical acceptance but rely on local heuristics for horizontal selection (\emph{horizontally limited}). \textbf{(c) UniVer} unifies both dimensions via Conditional OT, enabling parallel intra-layer mass allocation for different branches and seamless inter-layer mass propagation.}
\label{fig:overview}
\end{figure*}

This discrepancy raises a fundamental question: \emph{How to unify multi-draft selection and multi-step coordination into a single, principled verification framework?} The core difficulty lies in the combinatorial complexity of a candidate tree, where the exponential growth of potential paths makes direct joint optimal transport highly non-trivial. To bypass this complexity, we shift the focus from solving the global OT problem to a recursive local-global coordination. By introducing the \emph{prefix acceptance probability} of a node $v$ ($\tilde{p}_v$) as the dynamic scaling factor, we reformulate horizontal verification into a conditional optimal transport problem. This enables a unified strategy that couples multi-draft selection with cross-level dependencies while maintaining strict losslessness.

Guided by this principle, we propose UniVer (Unified Verification), a principled framework that decomposes tree-structured verification into a top-down allocation phase and a post-order decision phase. Unlike prior approaches~\cite{li2024eagle, weng2025traversal} that rely on Recursive Rejection Sampling without replacement, UniVer enables parallel computation of acceptance plans for all branches at the same depth. By leveraging conditional OT, our algorithm determines the acceptance mass of child nodes within a single probability space, effectively eliminating the sequential dependency inherent in RRSw. As a result, UniVer provides a lossless yet efficient paradigm for tree speculative decoding.

Our contributions are summarized as follows:
\begin{itemize}
    \item \textbf{Unified Framework.} We propose UniVer, a framework that unifies multi-step and multi-draft speculative decoding. By treating prefix acceptance probabilities $\tilde{p}_v$ as dynamic scaling factors, we cast tree-based verification as a series of conditional OT problems. The framework is not tied to a specific horizontal verifier and accommodates other single-step verification schemes as plug-in components.
    \item \textbf{Theoretical Foundation.} We prove that UniVer achieves the conditional optimum for any $\tilde{p}_v \in (0, 1]$ under a specific sampling strategy (\emph{i.e.,} Greedy selection~\cite{hu2025towards}), attaining the theoretical upper bound of acceptance mass under prefix constraints. We further show that this bound is strictly non-inferior to that of the Greedy method, establishing UniVer as a principled generalization of single-step OT optimality to multi-step tree verification.
    \item \textbf{Empirical Validation.} UniVer enables single-pass verification by eliminating the sequential dependencies inherent in RRSw and Traversal. Experiments across various LLMs and tree architectures show that UniVer consistently improves acceptance length by 4.2\% to 8.5\% over standard RRSw and by up to 2.5\% over Greedy under temperature 1.0 while maintaining losslessness.
\end{itemize}

\section{Preliminaries}

\subsection{Speculative Decoding and OT}
\label{sec:prelim-ot}
Let $\{\M_b,\M_s\}$ be a pair of probabilistic models over vocabulary $\Sigma$. For any given accepted prefix $\mathbf{x}$, we refer to $\M_b(\cdot\mid\mathbf{x})$ as the target distribution and $\M_s(\cdot\mid\mathbf{x})$ as the draft distribution, respectively.

\paragraph{Single-Draft Case.}
In standard speculative decoding~\cite{leviathan2023fast, chen2023accelerating}, a single draft token $x \sim \M_s(\cdot\mid\mathbf{x})$ is verified against the target distribution $\M_b$. The acceptance rate is determined by the probability of alignment between $\M_s$ and $\M_b$. Formally, this corresponds to an optimal transport problem, where the optimal acceptance rate $\alpha^*$ equals one minus the total variation distance between $\M_b$ and $\M_s$ (\emph{i.e.,} the probability mass shared by the target and draft distributions):
\begin{equation}
    \alpha^* = \sum_{x \in \Sigma} \min\{\M_b(x\mid\mathbf{x}), \M_s(x\mid\mathbf{x})\}.
\end{equation}

\paragraph{Local Heuristics for Multi-Draft.}
When $n$ draft tokens are generated, a straightforward extension is to apply the single-token verification sequentially. Methods such as Recursive Rejection Sampling (RRS) and its without replacement variant (RRSw)~\cite{miao2024specinfer, Chen2024Sequoia, Jeon2024RSD, Yang2024MCSD} adopt this strategy: they verify candidates one by one using local pairwise comparisons between $\M_b$ and $\M_s$, re-normalizing the residual distribution after each rejection. While RRSw prevents redundant sampling of identical tokens, it remains a \emph{local heuristic} that processes candidates independently rather than coordinating their acceptance globally.

\paragraph{Multi-Draft with Optimal Transport.}
In contrast, some recent works~\cite{sun2023spectr, hu2025towards, Khisti2025MultiDraftSS} frame multi-draft verification as an OT problem. \citet{hu2025towards} propose a Greedy method. They first deterministically select $m-1$ drafts with highest probabilities from $\M_s$, then sample the final draft from the residual $\M_s^\neg$. Under this sampling strategy, they derive a closed-form verification strategy that achieves the theoretical upper bound of the acceptance rate:
\begin{equation}\label{eq:greedy-acc}
\alpha^*_{\text{Greedy}}(\M_b, \M_s) = \sum_{i \in \text{Top}_{m-1}(\mathcal{M}_s)} \mathcal{M}_b(i) + \sum_{i \in \Sigma} \min\{\mathcal{M}_b(i), \mathcal{M}_s^{\neg}(i)\},
\end{equation}
where $\M_s^{\neg}$ denotes the residual draft distribution. Unlike RRSw, this OT-based approach works under the above sampling structure (${\rm Top}_{m-1}$ plus one residual sample) and coordinates the acceptance of all $m$ candidates in a unified probability space. Under such sampling scheme, Eq.~\eqref{eq:greedy-acc} maximizes the marginal acceptance probability at each step.

\subsection{Multi-step Strategies}
\label{sec:prelim-block}

Consider a draft model generating a candidate tree $T$. For clarity, we use $\mathbf{v} = (v_0, \dots, v_{\gamma})$ to represent a root-to-node chain $\mathbf{v}\subseteq T$, where $v_0$ denotes the root node and the chain length is $\gamma+1$.

\paragraph{Local Heuristics.}
The baseline approach applies vanilla verification at each step independently. For any fixed chain $\textbf{v}=(v_0, \dots, v_{\gamma})$, let the acceptance ratio be $r(v_i) = \M_b(v_i)/\M_s(v_i)$ at position $v_i$. Under local heuristics, the acceptance probability is computed independently at each position as $\min\{r(v_i), 1\}$. Consequently, the actual acceptance probability $p_{v_k}^{\text{Local}} $ of each prefix chain of $\textbf{v}$ is given by the product of individual, locally truncated probabilities:
\begin{equation*}
p_{v_k}^{\text{Local}} = \prod_{i=1}^{k} \min\{r(v_i), 1\},\;1\leq k\leq\gamma.
\end{equation*}
This \emph{post-hoc aggregation} treats each verification step as independent, failing to exploit the coupling of probability mass along the sequence, which leads to suboptimal cumulative acceptance rates.

\paragraph{Block Verification.}
For single-chain decoding, \citet{sun2025block} propose Block Verification. Unlike local verification, which independently truncates the local density ratio at $1.0$ before accumulation, Block Verification computes the prefix acceptance probability $p_{v_k}^{\text{Block}}$ by first performing cumulative multiplication and then applying truncation:
\begin{equation*}
p_{v_k}^{\text{Block}} = \min\left\{ r(v_k) \cdot p_{v_{k-1}}^{\text{Block}}, 1 \right\},\;1\leq k\leq \gamma,
\end{equation*}
where $p_{v_0}^{\rm Block} = 1$. This mechanism enables a high density ratio ($r(v_{k}) > 1$) to recharge the prefix probability diminished by prior steps. In contrast, local verification follows $p_{v_k}^{\text{Local}} = \min\{r(v_{k}), 1\} \cdot p_{v_{k-1}}^{\text{Local}}$, which permanently penalizes the chain for any local mismatch. Consequently, Block Verification yields a strictly non-inferior cumulative acceptance probability, as $p_{v_k}^{\text{Block}} \geq p_{v_k}^{\text{Local}}$ holds for all $1\leq k\leq \gamma$.

\paragraph{Traversal Verification.}
Traversal Verification \cite{weng2025traversal} extend Block Verification to tree structures through a \emph{post-order traversal} strategy. This approach leverages the joint probability for vertical optimization, but handles the horizontal dimension via RRSw. Specifically, Traversal Verification processes the tree bottom-up by iteratively verifying candidates against the residual target distribution, falling back to sibling nodes or the parent upon rejection. In other words, it lacks joint optimization across siblings, failing to exploit the horizontal coupling.

\section{Method}
\label{sec:method}

\subsection{Overview}
\label{sec:overview}
To bridge the gap between local horizontal selection and suboptimal vertical coordination in existing methods, we present UniVer, a unified framework for multi-draft and multi-step speculative decoding. As presented in Figure~\ref{fig:two-stage}, UniVer operates through a two-stage verification process that jointly optimizes horizontal and vertical optimization.

\noindent
\begin{minipage}[t]{0.43\textwidth}
\paragraph{Allocation Phase: Top-Down Probability Propagation.}
Starting from the root node with initial probability $\tilde{p}_{root} = 1$, we propagate the acceptance mass layer by layer. For each layer, the acceptance plans for all independent branches are computed in parallel. For each node $v$, we treat $\tilde{p}_v$ as the \emph{effective prefix acceptance probability} and compute the conditional optimal transport plan between the scaled target distribution $\tilde{p}_v \cdot \M_b(\cdot \mid v)$ and the draft distribution $\M_s(\cdot \mid v)$. This yields marginal acceptance probabilities $p_v(u)$ for each child $u \in \mathcal{C}(v)$, which are then normalized into conditional acceptance probabilities $\tilde{p}_u$ for the next layer. This propagation continues until all leaf nodes are processed.
\vspace{5pt}
\paragraph{Decision Phase: Post-Order Tree Traversal.}
With all acceptance probabilities pre-computed, we perform a standard post-order tree traversal (as defined in Traversal Verification~\cite{weng2025traversal}). For each node $v\in T$, we draw $\eta \sim U(0,1)$ and accept $v$ if $\eta < \tilde{p}_v$ (leaf) or $\eta < \tilde{p}_v^{\rm res}$ (non-leaf). If a node is rejected, we proceed to the next node in traversal order until an acceptance or all nodes are rejected. Upon acceptance at a leaf, we sample the next token from $\M_b$; upon acceptance at a non-leaf node, we sample from the residual distribution.
\end{minipage}
\hfill
\begin{minipage}[t]{0.53\textwidth}
\vspace{-18pt}
\begin{algorithm}[H]
\caption{UniVer}
\label{alg:univer}
\small
\begin{algorithmic}[1]
\Require Draft tree $T$ with root $r$, target model $\M_b$, draft model $\M_s$
\Ensure Accepted sequence $\mathbf{v}$ and next token $y$

\State // \textbf{Allocation Phase}
\State Initialize $\tilde{p}_r \gets 1$
\For{each layer $\ell = 0$ to $D-1$}
    \For{each node $v$ at layer $\ell$ (in parallel)}
        \State Compute marginal probabilities $p_v(u_i)$ for all $u_i \in \mathcal{C}(v)$ via Eqs.~\ref{eq:greedy-sample}--\ref{eq:greedy-res} \label{line:greedy}
        \State Compute rejection probability $p_v(\neg v)$ (Eq.~\ref{eq:reject})
        \State // Conditional normalization for traversal order
        \State $m \gets |\mathcal{C}(v)|$
        \For{$j = 1$ to $m$}
            \State $\tilde{p}_{u_j} \gets \frac{p_v(u_j)}{1 - \sum_{i=1}^{j-1} p_v(u_i)}$ \Comment{Conditional prob.} \label{line:cond}
        \EndFor
        \State $\tilde{p}_v^{\text{res}} \gets 1-\frac{p_v(\neg v)}{1 - \sum_{i=1}^{m} p_v(u_i)}$ \Comment{Fallback prob.} \label{line:residual}
    \EndFor
\EndFor

\State // \textbf{Decision Phase}
\State Generate $\boldsymbol{\eta} \sim U(0,1)^{|T|}$ %
\For{each node $v$ in post-order traversal sequence}
    \If{$v$ is leaf and $\eta_v < \tilde{p}_v$} %
        \State Sample $y \sim \M_b(\cdot \mid v)$ and \Return $(v, y)$
    \EndIf
    \If{$v$ is non-leaf and $\eta_v < \tilde{p}_v^{\rm res}$}
        \State Sample $y\in \Sigma \setminus \mathcal{C}(v)$ from normalized residual distribution
        $\frac{p_v(y)}{\sum_{u\notin\mathcal{C}(v)}p_v(u)}$
        and \Return $(v, y)$
    \EndIf
\EndFor
\end{algorithmic}
\end{algorithm}
\end{minipage}

\normalsize
\subsection{UniVer Algorithm}
\label{sec:univer}

We now introduce the technical details of the UniVer pipeline (Algorithm~\ref{alg:univer}), including its sampling strategy and computation of the core acceptance probability.

\begin{figure}[t]
\centering
\includegraphics[width=\textwidth]{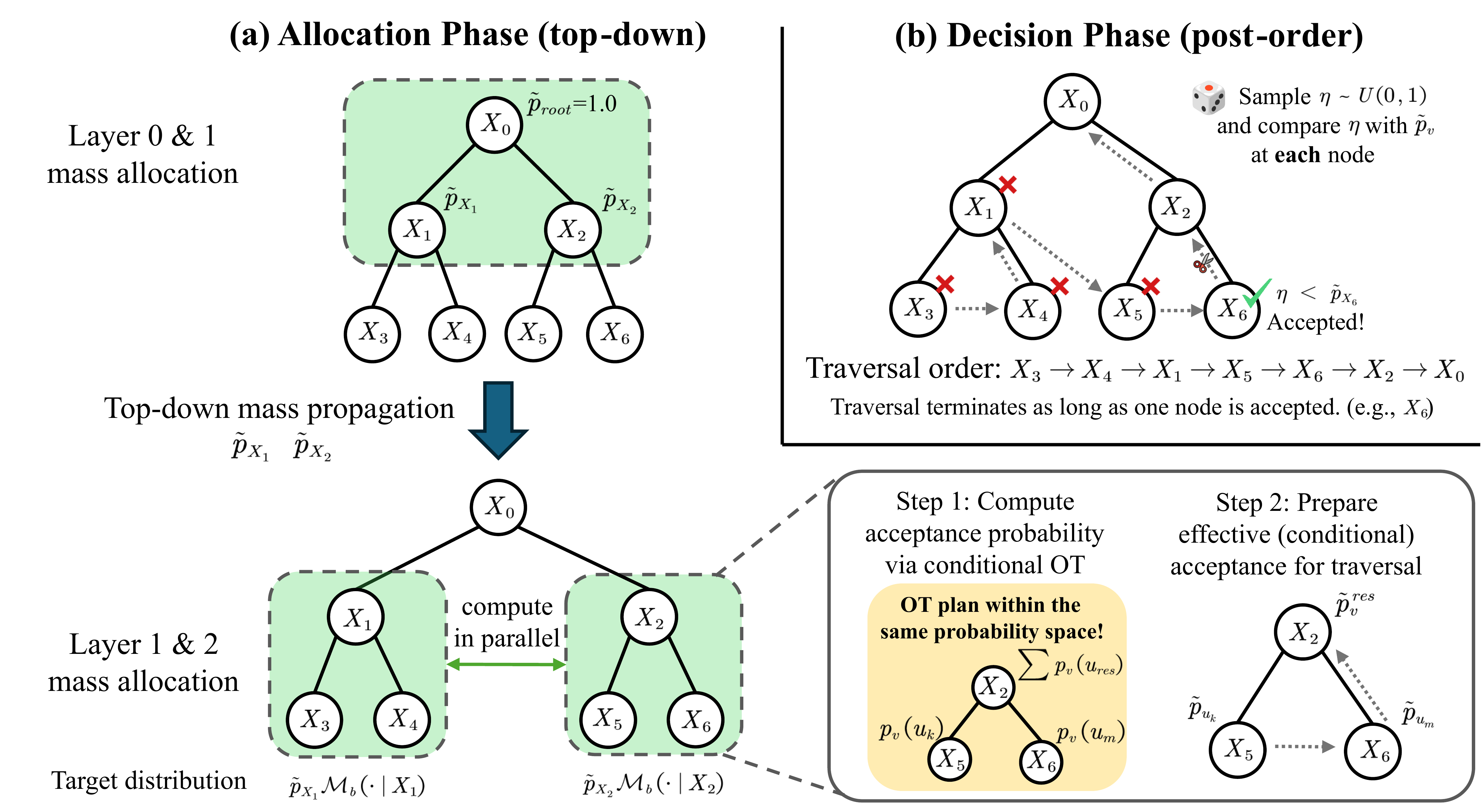}
\caption{Two-stage verification framework of \textbf{UniVer}. \textbf{(a) Allocation Phase}: Acceptance probabilities are propagated in a top-down and layer-by-layer manner. For each branch within a layer, marginal probabilities are computed in parallel using conditional optimal transport, then normalized into conditional probabilities used for the next layer's acceptance probability computation.
\textbf{(b) Decision Phase}: A single-pass post-order traversal is executed. For each node $X_i$, a random variable $\eta$ is compared against $\tilde{p}_{X_{i}}$ (for leaves) or $\tilde{p}_{X_{i}}^{\rm res}$ (for non-leaves); the process terminates upon acceptance.}
\label{fig:two-stage}
\end{figure}

\paragraph{UniVer Acceptance Probability Computation.}
For any parent node $v$ with effective conditional acceptance probability $\tilde{p}_v$ (precomputed recursively), its children nodes $\mathcal{C}(v) = \{u_1, \dots, u_m\}$ are selected by Greedy sampling strategy (i.e., ${\rm Top}_{m-1}$ plus $u_m\sim\M_s^\neg$).

For clarity, define the normalization factor:
\begin{equation}\label{eq:norm-factor}
Z_v = 1 - \tilde{p}_v + \sum_{x\in\Sigma}\bigl[\tilde{p}_v \M_b(x \mid v) - \M_s^\neg(x \mid v)\bigr]_+.
\end{equation}
The marginal acceptance probabilities are given by:
\begin{equation}\label{eq:greedy-sample}
p_v(u_m) = \min\left\{1, \frac{\tilde{p}_v \M_b(u_m \mid v)}{\M_s^\neg(u_m \mid v)} \right\}
\end{equation}
\begin{equation}\label{eq:greedy-top}
p_v(u_k) = \bigl[\tilde{p}_v \M_b(u_k \mid v) - \M_s^\neg(u_k \mid v)\bigr]_+ \times \frac{1-p_v(u_m)}{Z_v}, \quad \forall 1\le k\le m-1, 
\end{equation}
\begin{equation}\label{eq:greedy-res}
p_v(u_{\text{res}}) = \bigl[\tilde{p}_v \M_b(u_{\text{res}} \mid v) - \M_s^\neg(u_{\text{res}} \mid v)\bigr]_+ \times \frac{1-p_v(u_m)}{Z_v}, \quad \forall u_{\text{res}} \in \Sigma \setminus \mathcal{C}(v), 
\end{equation}
Then, the remain probability of rejecting $v$ is 
\begin{equation}\label{eq:reject}
p_v(\neg v) = 1 - \sum_{u\in\Sigma}p_v(u) =\frac{(1-p_v(u_m))(1 - \tilde{p}_v)}{Z_v}. 
\end{equation}
ensuring proper probability allocation.

Intuitively, Eqs.~\eqref{eq:norm-factor}--\eqref{eq:reject} first settle the sampled child $u_m$ as in standard speculative sampling, and then redistribute the probability of rejecting $u_m$ to the other outcomes in proportion to the target mass that $u_m$ cannot cover, \emph{i.e.,} $[\tilde{p}_v\M_b(x)-\M_s^\neg(x)]_+$, with $Z_v$ as the normalizer. For a non-root node, the unallotted mass $1-\tilde{p}_v$ is treated as a virtual outcome $\neg v$ (rejecting $v$ as a whole), so that accepting a child, falling back to $v$, and rejecting $v$ all lie in a single probability space; the root is the special case $\tilde{p}_r=1$. Appendix~\ref{sec:worked-example} explains this logic step by step on a two-layer example.

\newtheorem{theorem}{Theorem}
\newtheorem{definition}[theorem]{Definition}
\newtheorem{corollary}{Corollary}
\newtheorem{lemma}{Lemma}

\section{Theoretical Guarantees}
\label{sec:theory}

This section establishes the theoretical foundations of UniVer, proving three key properties:
\begin{itemize}
    \item \textbf{Losslessness.} The algorithm preserves the target model distribution $\M_b$.
    \item \textbf{Conditional Optimality.} UniVer achieves the optimal acceptance rate for each single-layer subtree of $T$, given a prefix acceptance probability and its sampling strategy.
    \item \textbf{Superiority over Greedy.} The expected acceptance length of our UniVer is never worse than that of the Greedy method \cite{hu2025towards}.
\end{itemize}

\paragraph{Preliminaries.}
We first formalize some key notations used in our analysis. Let $p_{\rm draft}$ be the probability distribution of generating a specific draft tree $T$ under the draft model $\M_s$, satisfying $\sum_T p_{\rm draft}(T) = 1$. For a verification algorithm $\mathscr{A}$, we write $\mathscr{A} = o'$ to indicate that the algorithm outputs exactly the sequence $o'$. Note that $o'$  except the last one form a valid root-to-node path in $T$, while its last element is a sample token $y$. Each \emph{prefix acceptance probability} $\tilde{p}_v \in [0,1]$ has been evaluated recursively by Line~10 of Algorithm \ref{alg:univer}.

Detailed proofs of the following theoretical results are referred to Appendix~\ref{sec:appendix_proofs}.

\begin{definition}[Locally Lossless Tree]
\label{def:loc-lossless}
Let $T$ be a draft tree equipped with acceptance probabilities $\{\tilde{p}_v\}_{v \in T}$ and $\{p_v(u)\}_{u \in \mathcal{C}(v)}$ generated by the \emph{Allocation Phase}. If for any non-leaf node $v \in T$, the following holds: given the prefix $v$ and $\tilde{p}_v$, the generation of children $\mathcal{C}(v)$ is \textbf{conditionally independent} of the part of $T$ outside the subtree rooted at $v$, and
\begin{equation}\label{eq:loc-lossless}
\mathbb{E}_{\mathcal{C}(v)}\!\left[p_v(t) \mid \tilde{p}_v\right] = \tilde{p}_v \M_b(t \mid v), \quad \forall t \in \Sigma,
\end{equation}
then $T$ is called a \textbf{locally lossless tree}.
\end{definition}

\begin{definition}[Lossless Verification]\label{def:lossless}
A verification algorithm $\mathscr{A}(T, \M_s, \M_b)$ is \textbf{lossless} if for any sequence $o = (o_0, o_1, \dots, o_L) \in \Sigma^{L+1}$ with $o_0=r$ being the root, the following equation holds:
\begin{equation}
\sum_{T} p_{\rm draft}(T) \sum_{o' \subseteq o} \frac{\M_b(o)\Pr[\mathscr{A} = o' \mid T]}{\M_b(o')}  = \M_b(o),
\label{eq:lossless-def}
\end{equation}
where the inner summation ranges over all prefix sequences $o' \subseteq o$ (including $o'=o$ itself), and $\frac{\M_b(o)}{\M_b(o')}$ is the conditional probability of generating the remaining tokens $\{o\setminus o'\}$ by $\M_b(\cdot\mid o')$.
\end{definition}

\begin{theorem}[Local losslessness of UniVer]\label{thm:loc-lossless-Uni}
The draft tree equipped with the acceptance probabilities generated in \emph{Allocation Phase} of UniVer verification is a locally lossless tree. 
\end{theorem}

\begin{theorem}[General Losslessness]\label{thm:general-lossless}
For any two-stage verification algorithm $\mathscr{A}(T, \M_s, \M_b)$ whose Decision Phase follows Algorithm~\ref{alg:univer}, if the draft tree generated in its Allocation Phase is \textbf{locally lossless}, then the verification $\mathscr{A}$ is \textbf{lossless}.
\end{theorem}

\begin{corollary}[Losslessness of UniVer]\label{thm:lossless}
From Theorems~\ref{thm:loc-lossless-Uni} and \ref{thm:general-lossless}, UniVer produces tokens are consistent with the exact target distribution $\M_b$, i.e., UniVer is \textbf{strictly lossless}.
\end{corollary}

\paragraph{Remark.} Theorem~\ref{thm:general-lossless} does not rely on the specific allocation rule in Eqs.~\eqref{eq:norm-factor}--\eqref{eq:reject}: any Allocation Phase that yields a locally lossless tree gives a lossless two-stage verifier. Therefore, other horizontal sampling and verification schemes can also be incorporated into UniVer by replacing these formulas accordingly. For example, using RRSw with sampling without replacement as the per-node primitive recovers Traversal Verification~\cite{weng2025traversal}. In contrast, the closed-form allocation, the conditional optimality in Theorem~\ref{thm:optimality}, and the superiority over Greedy in Theorem~6 below are specific to Greedy sampling.

\begin{theorem}[Conditional Optimality]\label{thm:optimality}
Let $T$ be any locally lossless tree satisfying Definition~\ref{def:loc-lossless} and let $v\in T$ be a non-leaf node with effective acceptance probability $\tilde{p}_v$. Then UniVer achieves the maximum conditional acceptance rate for the children $\mathcal{C}(v)=\{u_1,\dots,u_m\}$ (over the randomness of $u_m\sim\M_s^\neg(\cdot\mid v)$), which equals:
\begin{equation}\label{eq:cond-optimal}
\begin{split}
&\alpha^*_{\rm UniVer} (\tilde{p}_v\M_b(\cdot \mid v), \M_s(\cdot \mid v))\\= &
\sum_{u \in {\rm Top}_{m-1}(\M_s(\cdot \mid v))} \tilde{p}_v\M_b(u \mid v) + \sum_{u \in \Sigma} \min\{\tilde{p}_v\M_b(u \mid v), \M_s^{\neg}(u \mid v)\}.   
\end{split}
\end{equation}
\end{theorem}

\begin{theorem}\label{thm:univer_vs_greedy}
For any draft tree $T\sim p_{\rm draft}$, let $N_{\rm UniVer}(T)$ and $N_{\rm Greedy}(T)$ denote the acceptance lengths of UniVer and the Greedy method \cite{hu2025towards}, respectively. Then UniVer exhibits the following superiority:
\begin{equation}\label{eq:superiority}
\E_{T\sim p_{\rm draft}}[N_{\rm UniVer}(T)] \geq \E_{T\sim p_{\rm draft}}[N_{\rm Greedy}(T)].
\end{equation}
\end{theorem}

\section{Experiments}
\label{sec:experiments}

\subsection{Experimental Setup}

\paragraph{Datasets.} 
We conduct the experiments on Spec-Bench \cite{Specbench}, following the setup of Traversal Verification~\cite{weng2025traversal}. Spec-Bench encompasses six distinct domains, each with 80 representative samples: multi-turn conversation on MT-Bench~\cite{mt-bench}, translation on WMT14 DE-EN~\cite{wmt14}, summarization on CNN/Daily Mail~\cite{cnndm}, question answering on Natural Questions~\cite{natural_questions}, mathematical reasoning on GSM8K~\cite{gsm8k}, and retrieval-augmented generation on DPR~\cite{DPR}.

\begin{table}[t]
\centering
\caption{Acceptance length ($\tau$) and throughput on Vicuna-7B-v1.3. Autoregressive throughput baseline: 35.9 tokens/s. Best results are highlighted in \textbf{bold}.}
\small
\setlength{\tabcolsep}{2pt}
\begin{tabular}{@{}llcccccc|c|cc@{}}
\toprule
\multicolumn{2}{@{}l}{\textbf{Method}} & \textbf{MT} & \textbf{Trans.} & \textbf{Summ.} & \textbf{QA} & \textbf{Math} & \textbf{RAG} & \textbf{Avg. $\tau$ ($\Delta$)} & \textbf{Avg. TPS} \\
\midrule
\multirow{2}{*}{\textbf{RRSw-based}} 
 & RRSw & 3.42{\tiny ±0.03} & 2.73{\tiny ±0.03} & 2.97{\tiny ±0.02} & 2.72{\tiny ±0.01} & 3.47{\tiny ±0.06} & 3.05{\tiny ±0.01} & 3.06{\tiny ±0.01} (0.0\%) & 67.5{\tiny ±0.2} \\
 & Traversal & 3.49{\tiny ±0.03} & 2.75{\tiny ±0.04} & 3.03{\tiny ±0.02} & 2.79{\tiny ±0.03} & 3.55{\tiny ±0.06} & 3.15{\tiny ±0.06} & 3.13{\tiny ±0.02} (2.3\%$\uparrow$) & 68.9{\tiny ±0.5} \\
\midrule
\multirow{2}{*}{\textbf{OT-based}} 
 & Greedy & 3.56{\tiny ±0.04} & 2.80{\tiny ±0.03} & 3.12{\tiny ±0.03} & 2.87{\tiny ±0.05} & 3.66{\tiny ±0.06} & 3.24{\tiny ±0.03} & 3.21{\tiny ±0.02} (4.9\%$\uparrow$) & 70.4{\tiny ±0.6} \\
 & UniVer & \textbf{3.68{\tiny ±0.02}} & \textbf{2.83{\tiny ±0.04}} & \textbf{3.21{\tiny ±0.04}} & \textbf{2.94{\tiny ±0.02}} & \textbf{3.76{\tiny ±0.02}} & \textbf{3.32{\tiny ±0.05}} & \textbf{3.29{\tiny ±0.01}} (7.5\%$\uparrow$) & \textbf{72.2{\tiny ±0.5}} \\
\bottomrule
\end{tabular}
\label{tab:main-results}
\end{table}

\normalsize

\begin{table}[!t]
\centering
\caption{Acceptance length ($\tau$) and throughput on Llama3.1-8B-Instruct. Autoregressive throughput baseline: 34.5 tokens/s. Best results are highlighted in \textbf{bold}.}
\small
\setlength{\tabcolsep}{2pt}
\begin{tabular}{@{}llcccccc|c|cc@{}}
\toprule
\multicolumn{2}{@{}l}{\textbf{Method}} & \textbf{MT} & \textbf{Trans.} & \textbf{Summ.} & \textbf{QA} & \textbf{Math} & \textbf{RAG} & \textbf{Avg. $\tau$ ($\Delta$)} & \textbf{Avg. TPS} \\
\midrule
\multirow{2}{*}{\textbf{RRSw-based}} 
 & RRSw & 2.96{\tiny ±0.03} & 2.31{\tiny ±0.03} & 2.56{\tiny ±0.01} & 2.33{\tiny ±0.01} & 3.31{\tiny ±0.06} & 2.86{\tiny ±0.04} & 2.72{\tiny ±0.02} (0.0\%) & 48.5{\tiny ±0.2} \\
 & Traversal & 3.05{\tiny ±0.02} & 2.34{\tiny ±0.06} & 2.60{\tiny ±0.01} & 2.40{\tiny ±0.05} & 3.37{\tiny ±0.05} & 2.94{\tiny ±0.04} & 2.78{\tiny ±0.02} (2.2\%$\uparrow$) & 49.5{\tiny ±0.4} \\
\midrule
\multirow{2}{*}{\textbf{OT-based}} 
 & Greedy & 3.09{\tiny ±0.01} & 2.54{\tiny ±0.04} & 2.70{\tiny ±0.02} & 2.46{\tiny ±0.02} & 3.45{\tiny ±0.03} & 3.04{\tiny ±0.06} & 2.88{\tiny ±0.02} (5.9\%$\uparrow$) & 50.6{\tiny ±0.4} \\
 & UniVer & \textbf{3.19{\tiny ±0.02}} & \textbf{2.57{\tiny ±0.06}} & \textbf{2.77{\tiny ±0.01}} & \textbf{2.57{\tiny ±0.04}} & \textbf{3.51{\tiny ±0.03}} & \textbf{3.11{\tiny ±0.04}} & \textbf{2.95{\tiny ±0.01}} (8.5\%$\uparrow$) & \textbf{52.2{\tiny ±0.1}} \\
\bottomrule
\end{tabular}
\label{tab:llama3-results}
\end{table}
\normalsize

\paragraph{Models.} 
We mainly conduct experiments on the Vicuna~\cite{mt-bench} model family, using EAGLE~\cite{li2024eagle} as the draft model. This pairing is widely adopted in previous research~\cite{sun2024spechub, weng2025traversal, hu2025towards}. We also include experiments on other models such as Llama3.1-8B-Instruct~\cite{grattafiori2024llama3} to validate its generalization. Due to space limits, we do not include all results in the main text. Please refer to Appendix~\ref{sec:add_results} for comprehensive experiments with different drafter structures and model sizes.

\paragraph{Implementation and Comparisons.} 
We implement UniVer and baseline methods based on the EAGLE~\cite{li2024eagle} open-source repository. In UniVer's allocation phase, nodes at the same tree depth are processed in parallel using custom Triton kernels. All experiments are conducted on NVIDIA RTX A6000 48G GPUs, with each configuration running on three different random seeds. We report the \emph{acceptance length $\tau$} (average tokens generated per cycle) by mean$\pm$std as the primary metric, which reflects the theoretical efficiency of the verification algorithm. We also measure the practical throughput (Tokens Per Second, TPS) for a comprehensive comparison. We compare UniVer against existing verification methods spanning two paradigms:

\begin{enumerate}
    \item \textbf{RRSw-based:} (1) \textbf{\emph{Vanilla RRSw}}: Standard recursive rejection sampling without replacement applied independently at each node, as implemented in EAGLE~\cite{li2024eagle}; (2) \textbf{\emph{Traversal}}: The Traversal Verification algorithm~\cite{weng2025traversal} which extends Block Verification to trees using RRSw for horizontal draft selection.
    \item \textbf{OT-based:} (1) \textbf{\emph{Greedy method}}~\cite{hu2025towards}: Single-step OT with Greedy draft selection; (2) \textbf{\emph{UniVer}} (Ours): Our proposed unified framework that extends OT-based verification to tree structures.
\end{enumerate}

\subsection{Overall Performance}

Table~\ref{tab:main-results} presents the results on Vicuna-7B-v1.3 using a balanced binary tree of depth 5 (32 leaf nodes) under temperature 1.0. UniVer achieves a $\tau$ of 3.29, representing a 7.5\% improvement over the Vanilla RRSw baseline. Notably, Traversal Verification, which extends Block Verification to trees using RRSw, only marginally outperforms Vanilla RRSw (2.3\%), suggesting that bottom-up traversal alone provides limited benefit when horizontal draft selection remains suboptimal. In contrast, UniVer achieves a 7.5\% higher $\tau$ over the vanilla baseline and a 2.5\% gain over Greedy method~\cite{hu2025towards}.

\begin{figure}[!h]
\begin{minipage}[t]{0.53\linewidth}
\vspace{0pt}
Figure~\ref{fig:per-depth} demonstrates the per-depth acceptance rates on MT-bench. At depth 0, Greedy and UniVer both achieve 82.6\%, outperforming RRSw-based methods through coordinated horizontal selection. All methods exhibit a characteristic drop at depth 1, consistent with the known degradation of EAGLE draft model alignment beyond the first token \cite{zhang2025hass, weng2025coral}. As depth increases, RRSw degrades steadily while Traversal stabilizes around 75\% via vertical probability recharging but plateaus. Greedy enjoys strong early-layer acceptance, but shares similar decay patterns like RRSw. UniVer maintains an acceptance rate of at least 77\% at every depth, confirming that it successfully propagates the probability mass across depths and benefits from horizontal-vertical joint optimization.
\end{minipage}
\hfill
\begin{minipage}[t]{0.45\linewidth}
\vspace{0pt}
\centering

\includegraphics[width=\linewidth]{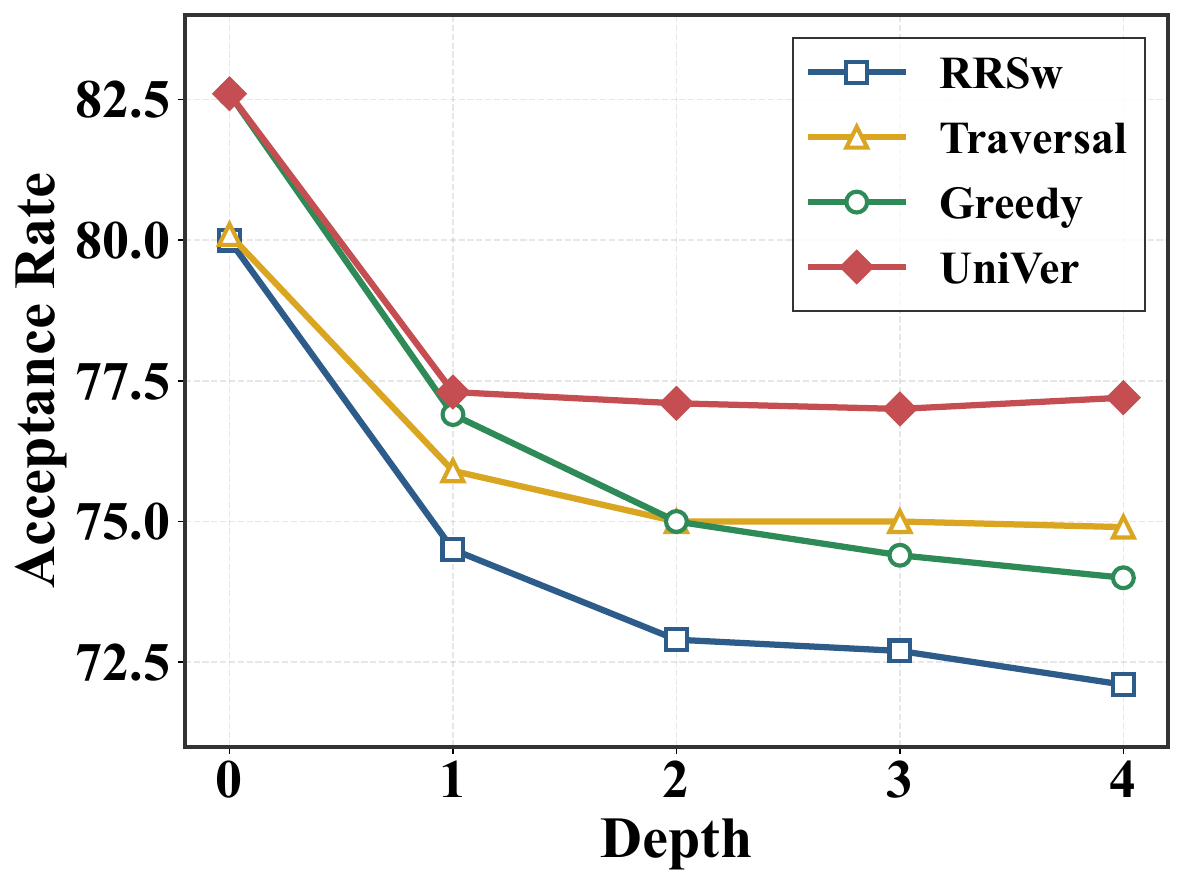}\caption{Per-depth acceptance rates on Vicuna-7B-v1.3 (binary tree, depth 5).}
\label{fig:per-depth}
\end{minipage}
\end{figure}

In terms of wall-clock efficiency, since UniVer only uses one sampled token with other tokens selected deterministically, the acceptance probabilities of all candidates admit closed-form expressions without using OT solvers. The overall latency of UniVer is therefore comparable to that of standard RRSw, and the acceptance gains directly translate into throughput improvements. To contextualize these numbers: for Vicuna-7B-v1.3, while UniVer improves acceptance length by 7.5\%, the end-to-end throughput gain is approximately 7\%. 

Similar trends are observed on Llama3.1-8B-Instruct (Table~\ref{tab:llama3-results}), confirming its robustness across different architectures. For full experimenetal results, please refer to Appendix~\ref{sec:add_results}.

\subsection{Scaling with Tree Size}

\begin{minipage}[t]{0.48\textwidth}
As shown in Figure~\ref{fig:scaling}, the advantage of UniVer becomes more pronounced as the tree size and depth increase. Compared to single-step OT (Greedy method), UniVer achieves larger gains when the tree gets deeper. This trend arises because UniVer's conditional OT framework effectively propagates probability mass through hierarchical dependencies, whereas Greedy applies independent single-step optimization at each node, as the local truncation of probabilities compounds across hierarchical levels. Similarly, as depth grows, Traversal demonstrates increasing advantages over Vanilla RRSw, validating that sequence-level verification benefits from deeper trees. However, UniVer consistently outperforms Traversal across all depths, as conditional OT achieves superior horizontal draft selection compared to RRSw's local pairwise comparisons.
\end{minipage}
\hfill
\begin{minipage}[t]{0.48\textwidth}
\vspace{-18pt}
\begin{figure}[H]
\centering
\includegraphics[width=\linewidth]{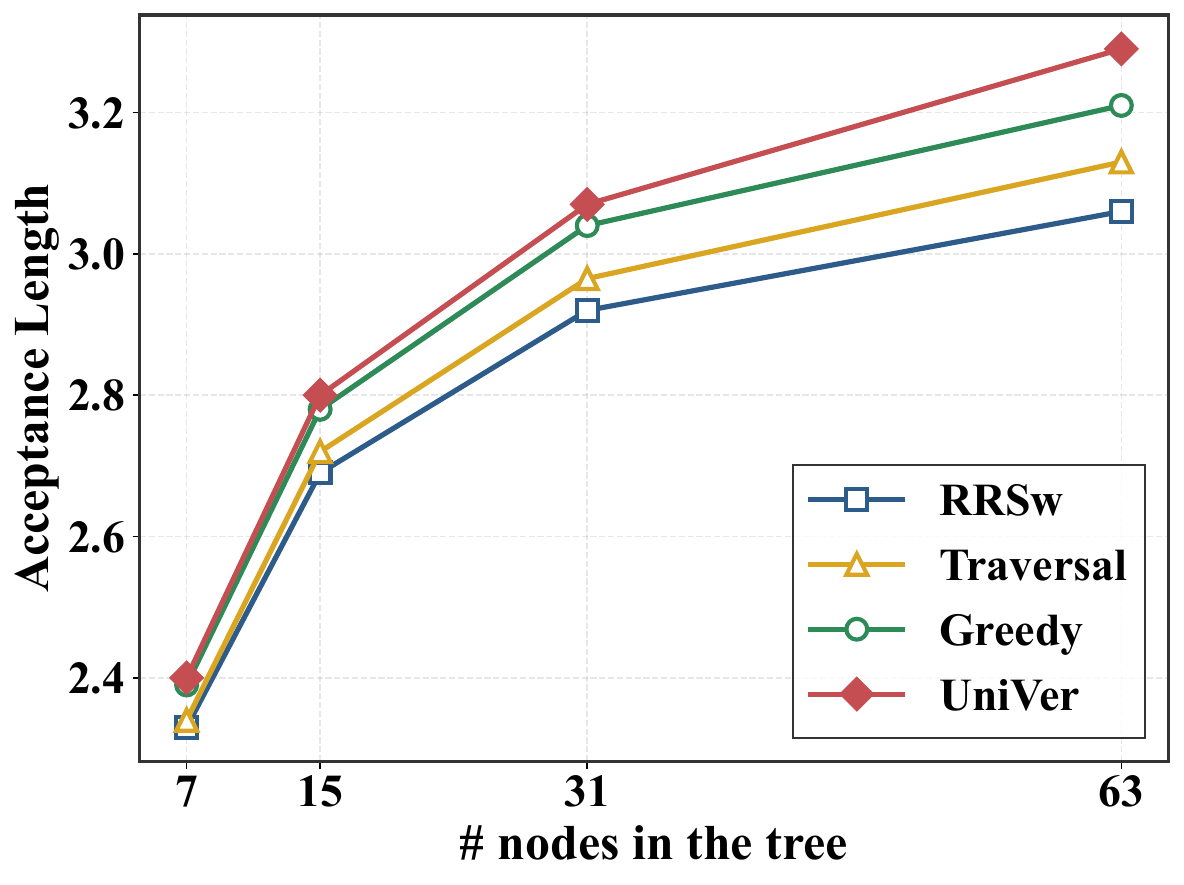}
\caption{Acceptance length as a function of binary tree size (number of nodes) on Vicuna-7B-v1.3. Tree depth ranges from 2 to 5, corresponding to 7, 15, 31, and 63 nodes respectively. The temperature is set to 1.0.}
\label{fig:scaling}
\end{figure}
\end{minipage}

\subsection{Effect of Tree Topology and Temperature} 

Table~\ref{tab:ablation-temp-tree} examines how UniVer performs across varying tree architectures and sampling temperatures. 

\paragraph{Impact of Temperature.} 
As the temperature decreases, the output distributions become more concentrated, causing the performance gap between verification methods to narrow. At temperature 0.3, Traversal, Greedy, and UniVer achieve nearly identical acceptance lengths (\emph{e.g.,} 3.39 for the binary tree with depth = 5 configuration), converging toward the same outcome as observed at temperature 0 (3.40), where the deterministic nature renders these verification methods equivalent. Conversely, at higher temperatures, where distributions are more dispersed, UniVer demonstrates more pronounced advantages over existing methods.

\paragraph{Impact of Tree Topology.} 
UniVer maintains consistent improvements across diverse tree structures, including balanced trees (2 drafts $\times$ 5 steps and 3 drafts $\times$ 3 steps) and the imbalanced EAGLE tree. While the absolute gain varies with topology (e.g., 7.5\% on the balanced 2x5 tree versus 4.2\% on the EAGLE tree at temperature 1.0), UniVer consistently achieves the highest acceptance lengths. This variation arises because the optimal transport formulation allocates probability mass globally across siblings, whereas the relative benefit of such global coordination depends on the specific branching structure and depth distribution of the tree.

\begin{table}[t]
\centering
\small
\caption{Impact of temperature and tree topology on acceptance length ($\tau$) using Vicuna-7B-v1.3. Results show mean{\tiny ±std}. $\Delta$ indicates relative improvement over Vanilla RRSw.}
\begin{tabular}{@{}ll|cc|cc|cccccc@{}}
\toprule
\multirow{2}{*}{\textbf{Method}} & & \multicolumn{2}{c|}{\textbf{Drafts=2 Steps=5}} & \multicolumn{2}{c|}{\textbf{Drafts=3 Steps=3}} & \multicolumn{2}{c}{\textbf{EAGLE Tree}} \\
& & Avg. $\tau$ & $\Delta$ & Avg. $\tau$ & $\Delta$ & Avg. $\tau$ & $\Delta$ \\
\midrule
\multicolumn{8}{c}{Temperature = 0.3} \\
\midrule
\multirow{2}{*}{\textbf{RRSw-based}} 
 & Vanilla RRSw & 3.34{\tiny ±0.01} & 0.0\% & 3.04{\tiny ±0.01} & 0.0\% & 3.33{\tiny ±0.01} & 0.0\% \\
 & Traversal & \textbf{3.39{\tiny ±0.01}} & 1.5\% & 3.08{\tiny ±0.01} & 1.3\% & \textbf{3.38{\tiny ±0.01}} & 1.5\% \\
\multirow{2}{*}{\textbf{OT-based}} 
 & Greedy & \textbf{3.39{\tiny ±0.01}} & 1.5\% & 3.08{\tiny ±0.01} & 1.3\% & \textbf{3.38{\tiny ±0.01}} & 1.5\% \\
 & UniVer & \textbf{3.39{\tiny ±0.01}} & 1.5\% & \textbf{3.09{\tiny ±0.01}} & 1.6\% & \textbf{3.38{\tiny ±0.01}} & 1.5\% \\
\midrule
\multicolumn{8}{c}{Temperature = 0.6} \\
\midrule
\multirow{2}{*}{\textbf{RRSw-based}} 
 & Vanilla RRSw & 3.29{\tiny ±0.01} & 0.0\% & 3.02{\tiny ±0.01} & 0.0\% & 3.29{\tiny ±0.01} & 0.0\% \\
 & Traversal & 3.32{\tiny ±0.01} & 0.9\% & 3.05{\tiny ±0.01} & 1.0\% & 3.33{\tiny ±0.01} & 1.2\% \\
\multirow{2}{*}{\textbf{OT-based}} 
 & Greedy & 3.34{\tiny ±0.02} & 1.5\% & 3.07{\tiny ±0.01} & 1.7\% & 3.34{\tiny ±0.01} & 1.5\% \\
 & UniVer & \textbf{3.37{\tiny ±0.01}} & 2.4\% & \textbf{3.09{\tiny ±0.01}} & 2.3\% & \textbf{3.35{\tiny ±0.01}} & 1.8\% \\
\midrule
\multicolumn{8}{c}{Temperature = 1.0} \\
\midrule
\multirow{2}{*}{\textbf{RRSw-based}} 
 & Vanilla RRSw & 3.06{\tiny ±0.01} & 0.0\% & 2.89{\tiny ±0.01} & 0.0\% & 3.11{\tiny ±0.01} & 0.0\% \\
 & Traversal & 3.13{\tiny ±0.02} & 2.3\% & 2.91{\tiny ±0.01} & 0.7\% & 3.16{\tiny ±0.01} & 1.6\% \\
\multirow{2}{*}{\textbf{OT-based}} 
 & Greedy & 3.21{\tiny ±0.02} & 4.9\% & 3.00{\tiny ±0.01} & 3.8\% & 3.21{\tiny ±0.01} & 3.2\% \\
 & UniVer & \textbf{3.29{\tiny ±0.01}} & 7.5\% & \textbf{3.03{\tiny ±0.01}} & 4.8\% & \textbf{3.24{\tiny ±0.02}} & 4.2\% \\
\midrule
\multicolumn{2}{c|}{Reference (Temperature = 0)} & 3.40 & - & 3.09 & - & 3.40 & - \\
\bottomrule
\end{tabular}
\label{tab:ablation-temp-tree}
\end{table}

\section{Related Work}

\textbf{Speculative Decoding.} 
Introduced by \citet{leviathan2023fast} and \citet{chen2023accelerating}, speculative decoding accelerates LLM inference via a draft-then-verify paradigm. Research in this area primarily focuses on two directions: (1) improving draft quality and tree structures, including architectural innovations such as EAGLE series~\cite{li2024eagle, li2025eagle3}, Medusa~\cite{cai2024medusa} and tree topology design, like EAGLE-2~\cite{LiEAGLE2} and Sequoia~\cite{Chen2024Sequoia}; (2) designing verification algorithms, which is most relevant to this work.

\textbf{Multi-Step Verification.}
Methods in this direction extend speculative decoding beyond single-step generation to handle sequential dependencies. For single-chain decoding, Block Verification~\cite{sun2025block} and ASpS~\cite{hu2024accelerated} identify the sub-optimality of local (per-step) verification and propose sequence-level alternatives that maximize expected acceptance length along the chain. \citet{weng2025traversal} further extend this to tree structures through Traversal Verification, which processes nodes bottom-up.

\textbf{Multi-Draft Verification.}
Complementary to multi-step methods, this direction focuses on optimizing the acceptance rate for multiple candidates generated at a single step. SpecTr~\cite{sun2023spectr} formulates verification as an optimal transport problem and has been further improved by subsequent works~\cite{Khisti2025MultiDraftSS, thomas2025global}. SpecInfer~\cite{miao2024specinfer} employs Recursive Rejection Sampling (RRS), later refined to RRSw~\cite{Jeon2024RSD,Yang2024MCSD,Chen2024Sequoia} to prevent redundant sampling of identical tokens. Beyond standard sampling strategies, SpecHub~\cite{sun2024spechub} and Greedy method~\cite{hu2025towards} propose hybrid approaches that deterministically select high-probability candidates with stochastic sampled tokens, with the latter achieving the theoretical optimal acceptance rate under its drafting strategy within single-step.

\section{Conclusion}
\label{sec:conclusion}

In this work, we addressed the fragmentation in speculative decoding verification, where existing methods treat multi-draft and multi-step optimizations in isolation. We revealed that longitudinal dependencies can be formulated as dynamic scaling factors via prefix acceptance probabilities and framed tree-based verification as a conditional OT problem. 

Building on this theoretical foundation, we proposed UniVer, an algorithm that composes local optimal transport plans under prefix constraints to jointly optimize across tree levels. We proved that UniVer achieves the optimal acceptance rate within this conditional framework while maintaining strict losslessness. Moreover, losslessness holds for any allocation that yields a locally lossless tree, so UniVer offers a general way to extend single-step horizontal verifiers to multi-step tree verification. Empirical results demonstrate that UniVer significantly improves acceptance length over existing methods, validating the efficacy of our unified verification strategy.

\paragraph{Limitations and Future Work.}
Distinct verification methods only exhibit performance differences when the sampling temperature is above zero. Therefore, UniVer possesses no extra performance gain under or near temperature = 0. Beyond this universal boundary, the realized acceptance length is intrinsically linked to the interplay between tree topology and sampling strategy. Although the UniVer framework admits other horizontal verifiers, its closed-form allocation and conditional optimality are derived under the specific mixed (Greedy) sampling strategy, while co-designing the sampling strategy, verification mechanism, and the tree structure to maximize longitudinal mass propagation remains a promising direction for future research. Moreover, although UniVer is optimal within the conditional framework, it does not always attain the theoretical upper bound given by the global joint-OT formulation (Appendix~\ref{sec:gap}). A better verifier therefore exists in principle, but how to construct it and whether its practical overhead is affordable remain open problems.

\begin{ack}
This work was supported by the University of Tokyo and Lenovo AI Technology Center. The authors declare no competing interests.
\end{ack}

\bibliographystyle{acl_natbib}  %
\small
\bibliography{Reference}
\normalsize

%%%%%%%%%%%%%%%%%%%%%%%%%%%%%%%%%%%%%%%%%%%%%%%%%%%%%%%%%%%%

\appendix

\section{A Worked Two-Layer Example of UniVer}
\label{sec:worked-example}

To make the two phases of UniVer concrete, we walk through a small two-layer tree, using the same toy distributions as the illustrative example of Traversal Verification~\cite{weng2025traversal}. Let $\Sigma=\{a,b,c\}$ and assume for simplicity that the distributions do not depend on the prefix: $\M_b=(0.3,0.4,0.3)$ and $\M_s=(0.6,0.3,0.1)$ over $(a,b,c)$. Each node expands $m=2$ children by Greedy sampling: the top-1 token $a$ is selected deterministically, and the second child is sampled from the residual draft distribution $\M_s^\neg=(0,0.75,0.25)$. We consider the tree in Figure~\ref{fig:worked-alloc}, where the sampled child is $b$ at every node. All numbers are rounded to three decimals, while intermediate computations use unrounded values.

\tikzset{
  wtok/.style={draw, rounded corners, align=center, inner sep=2.5pt, font=\small},
  wdone/.style={draw, rounded corners, fill=gray!10, text=gray, align=center, inner sep=2.5pt, font=\small},
  wedge/.style={draw, thick},
  wlab/.style={font=\footnotesize, fill=white, inner sep=1pt},
  wnote/.style={font=\footnotesize, align=left, text=black!75}}

\begin{figure}[h]
\centering
\begin{minipage}[b]{0.37\textwidth}
\centering
\scalebox{0.82}{%
\begin{tikzpicture}
\node[wtok] (r) at (0,0) {Root $X_0$\\ $\tilde{p}_r=1$,\ $Z_r=0.35$};
\node[wtok] (x1) at (-1.9,-2.0) {$X_1=a$\\ $\tilde{p}=0.400$};
\node[wtok] (x2) at (1.9,-2.0) {$X_2=b$\\ $\tilde{p}=0.889$};
\draw[wedge] (r) -- node[wlab, pos=0.5, above left] {$p_r(a)=0.400$} (x1);
\draw[wedge] (r) -- node[wlab, pos=0.5, above right] {$p_r(b)=0.533$} (x2);
\node[wnote] at (0,-3.25) {residual: $p_r(c)=0.067$,\ \ $p_r(\neg r)=0$\\ fallback: $\tilde{p}^{\rm res}_r=1$};
\end{tikzpicture}}
\par\smallskip {\small (a) Layer 0: allocation at the root}
\end{minipage}
\hfill
\begin{minipage}[b]{0.61\textwidth}
\centering
\scalebox{0.82}{%
\begin{tikzpicture}
\node[wdone] (x1) at (-2.5,0) {$X_1=a$\\ $\tilde{p}=0.400$,\ $Z=0.72$};
\node[wtok] (x3) at (-3.6,-2.0) {$X_3=a$\\ $\tilde{p}=0.131$};
\node[wtok] (x4) at (-1.4,-2.0) {$X_4=b$\\ $\tilde{p}=0.246$};
\draw[wedge] (x1) -- node[wlab, pos=0.5, left] {$0.131$} (x3);
\draw[wedge] (x1) -- node[wlab, pos=0.5, right] {$0.213$} (x4);
\node[wnote] at (-2.5,-3.25) {$p(c)=0$,\ \ $p(\neg X_1)=0.656$\\ fallback: $\tilde{p}^{\rm res}_{X_1}=0$};
\node[wdone] (x2) at (2.5,0) {$X_2=b$\\ $\tilde{p}=0.889$,\ $Z=0.394$};
\node[wtok] (x5) at (1.4,-2.0) {$X_5=a$\\ $\tilde{p}=0.356$};
\node[wtok] (x6) at (3.6,-2.0) {$X_6=b$\\ $\tilde{p}=0.736$};
\draw[wedge] (x2) -- node[wlab, pos=0.5, left] {$0.356$} (x5);
\draw[wedge] (x2) -- node[wlab, pos=0.5, right] {$0.474$} (x6);
\node[wnote] at (2.5,-3.25) {$p(c)=0.022$,\ \ $p(\neg X_2)=0.148$\\ fallback: $\tilde{p}^{\rm res}_{X_2}=0.130$};
\draw[gray, rounded corners] ([shift={(-0.25,-0.25)}]current bounding box.south west) rectangle ([shift={(0.25,0.25)}]current bounding box.north east);
\end{tikzpicture}}
\par\smallskip {\small (b) Layer 1: allocation at $X_1$ and $X_2$ (in parallel)}
\end{minipage}
\caption{Allocation Phase on a two-layer draft tree with $\M_b=(0.3,0.4,0.3)$ and $\M_s=(0.6,0.3,0.1)$ over $(a,b,c)$. Edge labels give the marginal acceptance probabilities $p_v(u)$, and each child node shows its conditional acceptance probability $\tilde{p}$ (Line~\ref{line:cond} of Algorithm~\ref{alg:univer}). (a) The root allocates the full target mass ($\tilde{p}_r=1$). (b) Each node in Layer~1 allocates its scaled target $\tilde{p}_v\M_b$ to its own children; nodes already processed in Layer~0 are shown in gray.}
\label{fig:worked-alloc}
\end{figure}

\paragraph{Allocation Phase.}
We explain the logic behind Eqs.~\eqref{eq:norm-factor}--\eqref{eq:reject} in four steps, using the tree in Figure~\ref{fig:worked-alloc} as a running example. These steps describe how the probabilities are allocated; the actual accept/reject decisions are made later in post-order.

\emph{Step 1: Settle the sampled child.}
Consider first the root (Figure~\ref{fig:worked-alloc}a), which is allotted the full target distribution ($\tilde{p}_r=1$). Once the top-$(m-1)$ children are fixed deterministically, the only randomness lies in the sampled child $u_m\sim\M_s^\neg$. As in standard speculative sampling, it is accepted with probability $p_r(u_m)=\min\{1,\M_b(u_m)/\M_s^\neg(u_m)\}$; here, $p_r(b)=\min\{1,0.4/0.75\}=0.533$.

\emph{Step 2: Redistribute the remaining mass in a unified probability space.}
For each token $x$, the target mass that the sampled child cannot cover is $[\M_b(x)-\M_s^\neg(x)]_+$. This includes the top-$(m-1)$ siblings, whose entire target mass remains since $\M_s^\neg(x)=0$, as well as the tokens outside the tree; here, it equals $(0.3,0,0.05)$ over $(a,b,c)$. The normalization factor
\begin{equation*}
Z_r=\sum_x[\M_b(x)-\M_s^\neg(x)]_+=1-\sum_x\min\{\M_b(x),\M_s^\neg(x)\}=0.35
\end{equation*}
is the total uncovered mass. It equals the expected probability that the sampled child is rejected, and it does not depend on which token is sampled (Eq.~\eqref{eq:Z-constant}). When $u_m$ is rejected, which happens with probability $1-p_r(u_m)$, this probability is split among the tokens in proportion to their uncovered mass, \emph{i.e.,} token $x$ receives the share $[\M_b(x)-\M_s^\neg(x)]_+/Z_r$ (Eqs.~\eqref{eq:greedy-top}--\eqref{eq:greedy-res}). Here, $a$ and $c$ take $0.3/0.35$ and $0.05/0.35$ of the rejection probability $1-0.533$, giving
\begin{equation*}
p_r(a)=\frac{0.3}{0.35}\times(1-0.533)=0.400,\qquad p_r(c)=\frac{0.05}{0.35}\times(1-0.533)=0.067.
\end{equation*}
As a result, accepting the sampled child, accepting a top-$(m-1)$ sibling, and falling back to the parent (\emph{i.e.,} accepting the parent and drawing a residual token such as $c$) all lie in a single probability space that sums to one. To see why this allocation is lossless, we average over the randomness of the sampled child. Each token $x$ receives mass $\M_s^\neg(x)\min\{1,\M_b(x)/\M_s^\neg(x)\}=\min\{\M_b(x),\M_s^\neg(x)\}$ through direct acceptance. Since the expected rejection probability of the sampled child is exactly $Z_r$, the expected mass redistributed to $x$ after rejection is exactly $[\M_b(x)-\M_s^\neg(x)]_+$. The two parts sum to the target mass,
\begin{equation*}
\underbrace{\min\{\M_b(x),\M_s^\neg(x)\}}_{\text{direct acceptance}}+\underbrace{[\M_b(x)-\M_s^\neg(x)]_+}_{\text{redistribution}}=\M_b(x),
\end{equation*}
so the local allocation preserves the target distribution.

\emph{Step 3: Extend to non-root nodes via an augmented target.}
A non-root node $v$ has a scaled target $\tilde{p}_v\M_b$, whose total mass is $\tilde{p}_v$. We assign the remaining mass $1-\tilde{p}_v$ to a virtual rejection outcome $\neg v$, forming a normalized augmented target distribution over $\Sigma\cup\{\neg v\}$. This specifies the rejection probability of $v$ in expectation over the sampled child; the rejection probability allocated for a particular sampled child can differ. Since $\M_s^\neg(\neg v)=0$, this virtual outcome can never be covered by the sampled child, just like a top-$(m-1)$ sibling. Applying Steps~1 and~2 to this augmented distribution recovers Eqs.~\eqref{eq:norm-factor}--\eqref{eq:reject}: the normalization factor $Z_v$ gains the additional term $1-\tilde{p}_v$, and the virtual outcome receives $p_v(\neg v)=(1-p_v(u_m))(1-\tilde{p}_v)/Z_v$. The root is simply the special case $\tilde{p}_r=1$. For instance, $X_1$ in Figure~\ref{fig:worked-alloc}b has $\tilde{p}_{X_1}=0.4$ and the augmented target $(0.12,0.16,0.12;\,0.6)$ over $(a,b,c;\,\neg X_1)$. Its sampled child $b$ is accepted with probability $\min\{1,0.16/0.75\}=0.213$, the uncovered mass is $(0.12,0,0;\,0.6)$ with $Z_{X_1}=0.72$, and the rejection probability $1-0.213$ is split into
\begin{equation*}
p_{X_1}(a)=\frac{0.12}{0.72}\times(1-0.213)=0.131,\qquad p_{X_1}(\neg X_1)=\frac{0.6}{0.72}\times(1-0.213)=0.656.
\end{equation*} Since only $40\%$ of the target mass is allotted to $X_1$, most of its rejection probability is allocated to $\neg X_1$ (here $0.656$ for the sampled child $b$, versus $0.6$ in expectation). In the Decision Phase, this outcome means that the subtree of $X_1$ is rejected and the traversal continues to the subtree of $X_2$. In parallel, $X_2$ with $\tilde{p}_{X_2}=0.889$ yields $p_{X_2}(b)=0.474$, $p_{X_2}(a)=0.356$, $p_{X_2}(c)=0.022$, and $p_{X_2}(\neg X_2)=0.148$.

\emph{Step 4: Convert marginal into conditional probabilities.}
The probabilities above are marginal within the probability space of $v$, whereas the Decision Phase visits the children one by one. Since a child $u_j$ is visited only after all its earlier siblings have been rejected, its marginal probability is converted into the conditional probability $\tilde{p}_{u_j}=p_v(u_j)/(1-\sum_{i<j}p_v(u_i))$ (Line~\ref{line:cond} of Algorithm~\ref{alg:univer}). Likewise, $v$ itself is reconsidered only after all its children are rejected, so its fallback probability $\tilde{p}_v^{\rm res}$ is the residual-token mass conditioned on this event (Line~\ref{line:residual}). This yields the values in Figure~\ref{fig:worked-alloc}:
\begin{align*}
&\text{root:} && \tilde{p}_{X_2}=\frac{0.533}{1-0.400}=0.889, && \tilde{p}_r^{\rm res}=1,\\
&X_1: && \tilde{p}_{X_4}=\frac{0.213}{1-0.131}=0.246, && \tilde{p}_{X_1}^{\rm res}=0,\\
&X_2: && \tilde{p}_{X_6}=\frac{0.474}{1-0.356}=0.736, && \tilde{p}_{X_2}^{\rm res}=\frac{0.022}{0.022+0.148}=0.130,
\end{align*}
where $\tilde{p}_{X_1}^{\rm res}=0$ because $X_1$ has no residual-token mass left. The conditional probability of a non-leaf child (\emph{e.g.,} $\tilde{p}_{X_1}=0.400$ and $\tilde{p}_{X_2}=0.889$) then serves as its scaling factor in the next layer.

\paragraph{Decision Phase.}
With the probabilities computed above, we visit the nodes in post-order: $X_3\to X_4\to X_1\to X_5\to X_6\to X_2\to X_0$. At each visited node $v$, we accept with probability $\tilde{p}_v$ for a leaf or $\tilde{p}_v^{\rm res}$ for a non-leaf node, stopping at the first acceptance. In this example, $X_1$ is always skipped because $\tilde{p}_{X_1}^{\rm res}=0$. Upon acceptance, we output the corresponding root-to-node prefix and sample one additional token from the target distribution at a leaf, or from the residual distribution at a non-leaf node.

\FloatBarrier
\section{Additional Experimental Results}
\label{sec:add_results}
We provide additional experimental results on Vicuna-13B-v1.3, Vicuna-33B-v1.3 and Qwen2-7B-Instruct with their corresponding EAGLE draft model in Table \ref{tab:v13b-results}, Table~\ref{tab:v33b-results}, and Table \ref{tab:qwen-results}. We also include experiments on Llama2-7B~\cite{llama2} with Llama-68M~\cite{miao2024specinfer} as the draft model in Table \ref{tab:llama2-results}.

\begin{table}[!h]
\centering
\caption{Acceptance length ($\tau$) and throughput on Vicuna-13B-v1.3. Autoregressive throughput baseline: 22.3 tokens/s. Best results are highlighted in \textbf{bold}.}
\small
\setlength{\tabcolsep}{2pt}
\begin{tabular}{@{}llcccccc|c|cc@{}}
\toprule
\multicolumn{2}{@{}l}{\textbf{Method}} & \textbf{MT} & \textbf{Trans.} & \textbf{Summ.} & \textbf{QA} & \textbf{Math} & \textbf{RAG} & \textbf{Avg. $\tau$ ($\Delta$)} & \textbf{Avg. TPS} \\
\midrule
\multirow{2}{*}{\textbf{RRSw-based}} 
 & RRSw & 3.57{\tiny ±0.06} & 2.81{\tiny ±0.03} & 3.17{\tiny ±0.03} & 2.77{\tiny ±0.05} & 3.63{\tiny ±0.08} & 3.11{\tiny ±0.05} & 3.17{\tiny ±0.02} (0.0\%) & 44.3{\tiny ±0.2} \\
 & Traversal & 3.60{\tiny ±0.05} & 2.82{\tiny ±0.02} & 3.19{\tiny ±0.03} & 2.85{\tiny ±0.05} & 3.71{\tiny ±0.05} & 3.15{\tiny ±0.04} & 3.22{\tiny ±0.01} (1.6\%$\uparrow$) & 44.9{\tiny ±0.1} \\
\midrule
\multirow{2}{*}{\textbf{OT-based}} 
 & Greedy & 3.67{\tiny ±0.02} & 2.92{\tiny ±0.04} & 3.27{\tiny ±0.01} & 2.91{\tiny ±0.04} & 3.79{\tiny ±0.03} & 3.28{\tiny ±0.04} & 3.30{\tiny ±0.02} (4.1\%$\uparrow$) & 45.7{\tiny ±0.3} \\
 & UniVer & \textbf{3.74{\tiny ±0.02}} & \textbf{2.97{\tiny ±0.03}} & \textbf{3.31{\tiny ±0.03}} & \textbf{2.97{\tiny ±0.01}} & \textbf{3.84{\tiny ±0.10}} & \textbf{3.30{\tiny ±0.03}} & \textbf{3.35{\tiny ±0.01}} (5.7\%$\uparrow$) & \textbf{46.5{\tiny ±0.4}} \\
\bottomrule
\end{tabular}
\label{tab:v13b-results}
\end{table}

\begin{table}[!h]
\centering
\caption{Acceptance length ($\tau$) and throughput on Vicuna-33B-v1.3. Autoregressive throughput baseline: 9.2 tokens/s. Best results are highlighted in \textbf{bold}.}
\small
\setlength{\tabcolsep}{2pt}
\begin{tabular}{@{}llcccccc|c|cc@{}}
\toprule
\multicolumn{2}{@{}l}{\textbf{Method}} & \textbf{MT} & \textbf{Trans.} & \textbf{Summ.} & \textbf{QA} & \textbf{Math} & \textbf{RAG} & \textbf{Avg. $\tau$ ($\Delta$)} & \textbf{Avg. TPS} \\
\midrule
\multirow{2}{*}{\textbf{RRSw-based}} 
 & RRSw & 3.38{\tiny ±0.04} & 2.68{\tiny ±0.01} & 2.97{\tiny ±0.01} & 2.66{\tiny ±0.02} & 3.70{\tiny ±0.01} & 2.94{\tiny ±0.02} & 3.05{\tiny ±0.02} (0.0\%) & 20.6{\tiny ±0.1} \\
 & Traversal & 3.41{\tiny ±0.01} & 2.72{\tiny ±0.01} & 3.01{\tiny ±0.01} & 2.68{\tiny ±0.04} & 3.73{\tiny ±0.04} & 2.98{\tiny ±0.03} & 3.09{\tiny ±0.01} (1.1\%$\uparrow$) & 20.8{\tiny ±0.1} \\
\midrule
\multirow{2}{*}{\textbf{OT-based}} 
 & Greedy & 3.50{\tiny ±0.01} & 2.80{\tiny ±0.02} & 3.07{\tiny ±0.01} & 2.78{\tiny ±0.01} & 3.82{\tiny ±0.03} & 3.05{\tiny ±0.03} & 3.17{\tiny ±0.01} (3.8\%$\uparrow$) & 21.2{\tiny ±0.1} \\
 & UniVer & \textbf{3.52{\tiny ±0.01}} & \textbf{2.83{\tiny ±0.02}} & \textbf{3.11{\tiny ±0.03}} & \textbf{2.81{\tiny ±0.06}} & \textbf{3.86{\tiny ±0.01}} & \textbf{3.10{\tiny ±0.05}} & \textbf{3.21{\tiny ±0.01}} (4.9\%$\uparrow$) & \textbf{21.4{\tiny ±0.1}} \\
\bottomrule
\end{tabular}
\label{tab:v33b-results}
\end{table}

\begin{table}[!h]
\centering
\caption{Acceptance length ($\tau$) and throughput on Qwen2-7B-Instruct. Autoregressive throughput baseline: 37.0 tokens/s. Best results are highlighted in \textbf{bold}.}
\small
\setlength{\tabcolsep}{2pt}
\begin{tabular}{@{}llcccccc|c|cc@{}}
\toprule
\multicolumn{2}{@{}l}{\textbf{Method}} & \textbf{MT} & \textbf{Trans.} & \textbf{Summ.} & \textbf{QA} & \textbf{Math} & \textbf{RAG} & \textbf{Avg. $\tau$ ($\Delta$)} & \textbf{Avg. TPS} \\
\midrule
\multirow{2}{*}{\textbf{RRSw-based}} 
 & RRSw & 2.09{\tiny ±0.01} & 2.57{\tiny ±0.03} & 1.44{\tiny ±0.01} & 2.37{\tiny ±0.03} & 3.09{\tiny ±0.02} & 1.48{\tiny ±0.02} & 2.17{\tiny ±0.01} (0.0\%) & 43.7{\tiny ±0.3} \\
 & Traversal & 2.14{\tiny ±0.01} & 2.61{\tiny ±0.03} & 1.45{\tiny ±0.01} & 2.43{\tiny ±0.04} & 3.12{\tiny ±0.03} & 1.50{\tiny ±0.02} & 2.21{\tiny ±0.01} (1.8\%$\uparrow$) & 44.0{\tiny ±0.3} \\
\midrule
\multirow{2}{*}{\textbf{OT-based}} 
 & Greedy & 2.15{\tiny ±0.02} & 2.64{\tiny ±0.04} & 1.49{\tiny ±0.02} & 2.46{\tiny ±0.05} & 3.21{\tiny ±0.01} & 1.53{\tiny ±0.02} & 2.25{\tiny ±0.02} (3.7\%$\uparrow$) & 44.7{\tiny ±0.3} \\
 & UniVer & \textbf{2.20{\tiny ±0.01}} & \textbf{2.70{\tiny ±0.04}} & \textbf{1.50{\tiny ±0.01}} & \textbf{2.57{\tiny ±0.04}} & \textbf{3.26{\tiny ±0.02}} & \textbf{1.56{\tiny ±0.02}} & \textbf{2.30{\tiny ±0.01}} (6.0\%$\uparrow$) & \textbf{45.8{\tiny ±0.2}} \\
\bottomrule
\end{tabular}
\label{tab:qwen-results}
\end{table}

\begin{table}[!h]
\centering
\caption{Acceptance length ($\tau$) and throughput on Llama2-7B. Autoregressive throughput baseline: 37.3 tokens/s. Best results are highlighted in \textbf{bold}.}
\small
\setlength{\tabcolsep}{2pt}
\begin{tabular}{@{}llcccccc|c|cc@{}}
\toprule
\multicolumn{2}{@{}l}{\textbf{Method}} & \textbf{MT} & \textbf{Trans.} & \textbf{Summ.} & \textbf{QA} & \textbf{Math} & \textbf{RAG} & \textbf{Avg. $\tau$ ($\Delta$)} & \textbf{Avg. TPS} \\
\midrule
\multirow{2}{*}{\textbf{RRSw-based}} 
 & RRSw & 2.47{\tiny ±0.03} & 2.41{\tiny ±0.04} & 2.15{\tiny ±0.03} & 2.60{\tiny ±0.02} & 2.45{\tiny ±0.06} & 2.52{\tiny ±0.11} & 2.43{\tiny ±0.04} (0.0\%) & 57.9{\tiny ±1.3} \\
 & Traversal & 2.55{\tiny ±0.04} & 2.49{\tiny ±0.09} & 2.22{\tiny ±0.04} & 2.65{\tiny ±0.06} & 2.58{\tiny ±0.06} & 2.69{\tiny ±0.02} & 2.53{\tiny ±0.03} (4.1\%$\uparrow$) & 59.6{\tiny ±1.1} \\
\midrule
\multirow{2}{*}{\textbf{OT-based}} 
 & Greedy & 2.48{\tiny ±0.04} & 2.52{\tiny ±0.07} & 2.27{\tiny ±0.03} & 2.68{\tiny ±0.04} & 2.53{\tiny ±0.06} & 2.71{\tiny ±0.04} & 2.53{\tiny ±0.03} (4.1\%$\uparrow$) & 59.1{\tiny ±1.0} \\
 & UniVer & \textbf{2.73{\tiny ±0.01}} & \textbf{2.59{\tiny ±0.09}} & \textbf{2.31{\tiny ±0.04}} & \textbf{2.84{\tiny ±0.04}} & \textbf{2.66{\tiny ±0.04}} & \textbf{2.77{\tiny ±0.07}} & \textbf{2.65{\tiny ±0.03}} (9.1\%$\uparrow$) & \textbf{62.5{\tiny ±0.4}} \\
\bottomrule
\end{tabular}
\label{tab:llama2-results}
\end{table}

\FloatBarrier

\section{Gap to the Global Joint-OT Optimum}
\label{sec:gap}

Theorem~\ref{thm:optimality} establishes that UniVer is optimal under the conditional framework, yet its top-down decomposition need not attain the global joint-OT optimum over the whole tree. Here we quantify this gap. Using the notation of Section~\ref{sec:theory}, under complete information the tree verification problem can be written as the following linear program (LP):
\begin{equation}\label{eq:lp}
\begin{aligned}
\max_{S_{o,T}\ge 0}\ \ & \sum_{o,T}(|o|-2)\,S_{o,T}\\
\text{s.t.}\ \ & \sum_{o'\subseteq o}\frac{\M_b(o)}{\M_b(o')}\sum_{T} S_{o',T}\le \M_b(o)\ \ (\forall o),\qquad
\sum_{o} S_{o,T}\le p_{\rm draft}(T)\ \ (\forall T),
\end{aligned}
\end{equation}
where $S_{o,T}$ denotes the probability that tree $T$ is sampled and the verifier outputs $o$, and $|o|-2$ is the number of accepted draft tokens (excluding the root $o_0=r$ and the final token sampled from the target model, consistent with $N_{\mathscr A}(T)$ in Appendix~\ref{sec:appendix_proofs}). For single-step drafts, Eq.~\eqref{eq:lp} degenerates to the formulation of \citet{hu2025towards}. Every lossless verifier (including UniVer) induces a feasible solution of Eq.~\eqref{eq:lp}, so its optimum upper-bounds the best achievable acceptance length, while UniVer provides a constructive lower bound:
\begin{equation*}
\text{UniVer} \;\le\; \text{Global joint-OT optimum} \;\le\; \text{LP value}.
\end{equation*}
Solving Eq.~\eqref{eq:lp} is intractable in general, since its state space grows with the tree depth and the number of nodes. In a minimal two-layer example, however, the bound can be computed in closed form.

\paragraph{Setup.}
Let $\Sigma=\{a,b,c\}$ with root context $r$ and distributions over $(a,b,c)$:
\begin{itemize}
    \item Target: $\M_b(\cdot\mid r)=(0.6,0.2,0.2)$, $\M_b(\cdot\mid r,b)=(0.05,0.05,0.9)$, $\M_b(\cdot\mid r,c)=(0.5,0.25,0.25)$;
    \item Draft: $\M_s(\cdot\mid r)=(0.7,0.15,0.15)$, $\M_s(\cdot\mid r,b)=\M_s(\cdot\mid r,c)=(0.5,0.25,0.25)$.
\end{itemize}
The tree is built by Greedy sampling with $m=2$. At depth $1$, the root expands $a$ (top-1, leaf) and $s_1\sim{\rm Unif}\{b,c\}$ (sampled, non-leaf); at depth $2$, $s_1$ expands $a$ (top-1, leaf) and $s_2\sim{\rm Unif}\{b,c\}$ (sampled, leaf). This yields four equally likely trees $T_1,\dots,T_4$ with $(s_1,s_2)\in\{(b,b),(b,c),(c,b),(c,c)\}$, each with probability $1/4$.

\paragraph{Analysis.}
The chain $(b,c)$ exists only in $T_2$ and carries target mass $\M_b(bc)=0.2\times0.9=0.18$. At the root, Eqs.~\eqref{eq:norm-factor}--\eqref{eq:reject} give $p_r(a)=0.6$ and $p_r(b)=0.4$, so the subtree of $b$ is entered with probability $0.4$ and conditional probability $\tilde{p}_b=1$. The accepted mass of UniVer on this chain is therefore
\begin{equation*}
A_{\rm UniVer}(bc)=\Pr[T_2]\times p_r(b)\times\min\Bigl\{1,\tilde{p}_b\,\frac{\M_b(c\mid r,b)}{\M_s^\neg(c\mid r,b)}\Bigr\}=\frac{1}{4}\times0.4\times1=0.10,
\end{equation*}
whereas any lossless verifier satisfies $A(bc)\le\min\{\M_b(bc),\Pr[T_2]\}=0.18$. This bound is attained by accepting the chain with conditional probability $0.72$ within $T_2$. On all other chains, UniVer already matches the upper bound (Table~\ref{tab:gap-example}).

\begin{table}[h]
\centering
\caption{Marginal accepted mass of each depth-2 chain and the expected number of accepted draft tokens $\E[N]$ in the minimal two-layer example. Depth-1 tokens are always accepted, contributing $1$ to $\E[N]$.}
\label{tab:gap-example}
\small
\begin{tabular}{@{}lcccccc|c@{}}
\toprule
Chain $(s_1,y)$ & $(b,a)$ & $(b,b)$ & $(b,c)$ & $(c,a)$ & $(c,b)$ & $(c,c)$ & $\E[N]$ \\
\midrule
Target mass $\M_b(s_1y)$ & 0.01 & 0.01 & 0.18 & 0.10 & 0.05 & 0.05 & -- \\
UniVer $A(s_1y)$ & 0.01 & 0.01 & \textbf{0.10} & 0.10 & 0.05 & 0.05 & 1.320 \\
Upper bound (LP) & 0.01 & 0.01 & \textbf{0.18} & 0.10 & 0.05 & 0.05 & 1.400 \\
\bottomrule
\end{tabular}
\end{table}

UniVer thus attains $1.320/1.400=94.3\%$ of the globally optimal expected accepted length. In this example the LP upper bound is tight: it can be attained by propagating the untruncated scaling factor $0.9/0.5=1.8$ into the next allocation step rather than clipping it at $1$. This factor should be interpreted as an auxiliary scaling weight rather than a valid acceptance probability. Such over-unity weights fall outside the current UniVer framework, and extending them to general branching trees while preserving shared-prefix marginals is nontrivial.

\section{The Proof of Theorems}
\label{sec:appendix_proofs}

\paragraph{Vocabulary and Sequences.}
Let $\Sigma$ denote a finite vocabulary. For any positive integer $L$, we use $\Sigma^L$ denote the set of all sequences (strings) of length $L$ over $\Sigma$. We write $a = (a_0, \dots, a_L) \in \Sigma^{L+1}$ to represent a sequence of length $|a| = L+1$, and $a_{[:\ell]}$ to represent the prefix $(a_0, \dots, a_\ell)$ consisting of the first $\ell$ elements of $a$.

\subsection{Local losslessness of UniVer}

\begin{proof}[Proof of Theorem \ref{thm:loc-lossless-Uni}]
Considering the draft tree $T$ generated by \emph{Allocation Phase} of UniVer verification (see Section~\ref{sec:univer}), we observe that, for any fixed non-leaf node $v\in T$, the randomness of its children set $\mathcal{C}(v)=\{u_1,\dots,u_m\}$ only depends on the last sampling token $u_m\sim \M_s^\neg(\cdot\mid v)$, which means $\Pr(\text{draft }\mathcal{C}(v))= \M_s^\neg(u_m\mid v)$. 
Throughout this proof, we write $\mathbb{E}_{\mathcal{C}(v)}[\cdot]$ as shorthand for $\mathbb{E}_{\mathcal{C}(v)}[\cdot \mid \tilde{p}_v]$, since $\tilde{p}_v$ is fixed when processing node $v$ and the only remaining randomness is in the sampling of $u_m$.

Now we begin to prove $\mathbb{E}_{\mathcal{C}(v)}[p_v(t)] = \tilde{p}_v \M_b(t|v)$ holds for all $t\in \Sigma$.
First, observe that the normalization factor $Z_v$ in Equations~\eqref{eq:norm-factor} is actually independent of the sampled value $u_m$. We use $H_v=\{u_1,\dots,u_{m-1}\}$ represent the top-$(m-1)$ tokens, and then
\begin{align}
Z_v &= 1 - \tilde{p}_v + \sum_{x\in\Sigma}[\tilde{p}_v \M_b(x|v) - \M_s^\neg(x|v)]_+ \notag\\
&= 1 - \tilde{p}_v + \sum_{x\in H_v}\tilde{p}_v \M_b(x|v) + \sum_{x\notin H_v}[\tilde{p}_v \M_b(x|v) - \M_s^\neg(x|v)]_+ \notag\\
&= 1 - \tilde{p}_v\sum_{x\notin H_v}\M_b(x|v) + \sum_{x\notin H_v}[\tilde{p}_v \M_b(x|v) - \M_s^\neg(x|v)]_+\notag\\
&= 1 - \sum_{x\notin H_v}\min\{\tilde{p}_v \M_b(x|v), \M_s^\neg(x|v)\}. \label{eq:Z-constant}
\end{align}
For any $t\in \Sigma$,
we consider the following two cases.

\textbf{Case 1: $t \in H_v$ (deterministic nodes).} 
From Equation~\eqref{eq:greedy-top} and noting that $\M_s^\neg(t\mid v)=0$ for $t\in H_v$, we have
$$\mathbb{E}_{\mathcal{C}(v)}[p_v(t)] 
=\frac{\tilde{p}_v \M_b(t\mid v)}{Z_v}\cdot\mathbb{E}_{u_m}[1-p_v(u_m)].$$
By Equation \eqref{eq:greedy-sample}, we know
\begin{equation}
\mathbb{E}_{u_m}[1-p_v(u_m)] = 1 - \sum_{x\notin H_v}\min\left\{\M_s^\neg(x|v), \tilde{p}_v \M_b(x|v)\right\}. \notag
\end{equation}
The expectation term equals $Z_v$ (see Equation~\eqref{eq:Z-constant}). Thus $\mathbb{E}_{u_m}[p_v(t)] = \tilde{p}_v \M_b(t\mid v)$.
	
\textbf{Case 2: $t \notin H_v$ (residual vocabulary).}
Here $t$ may either be the sampled node $u_m$ or belong to the remaining residual set. Obviously,
\begin{align*}
&\mathbb{E}_{u_m}[p_v(t)] = \M_s^\neg(t\mid v)\cdot\min\left\{1, \frac{\tilde{p}_v \M_b(t\mid v)}{\M_s^\neg(t\mid v)}\right\} \\
&+ \frac{[\tilde{p}_v \M_b(t \mid v) - \M_s^\neg(t\mid v)]_+}{Z_v}\times \sum_{x\notin H_v \cup \{t\}}\M_s^\neg(x\mid v)\cdot (1-p_v(x)).
\end{align*}
Since 
\begin{align*}
&\sum_{x\notin H_v \cup \{t\}}\M_s^\neg(x\mid v)(1-p_v(x)) \\
=\;& \sum_{x\notin H_v \cup \{t\}} \M_s^\neg(x\mid v) - \sum_{x\notin H_v \cup \{t\}} \min\{\M_s^\neg(x\mid v), \tilde{p}_v \M_b(x\mid v)\}\\ 
=\;& Z_v + \min\{0, \tilde{p}_v \M_b(t\mid v) - \M_s^\neg(t\mid v)\}.
\end{align*}
Substituting back, we obtain
\begin{align*}
\mathbb{E}_{u_m}[p_v(t)] =\;& \min\left\{\M_s^\neg(t\mid v), \tilde{p}_v \M_b(t\mid v)\right\} + [\tilde{p}_v \M_b(t\mid v) - \M_s^\neg(t\mid v)]_+\\
=\;& \tilde{p}_v \M_b(t\mid v).
\end{align*}
Therefore, the draft tree generated in UniVer verification is a locally lossless tree.
\end{proof}

\subsection{Losslessness of two-stage verification}
\begin{proof}[Proof of Theorem \ref{thm:general-lossless}]
We prove it by induction on the number of parent nodes $n$ in tree $T$.

\paragraph{Base case ($n=1$):} The tree $T$ consists of only the root node $r$ and its children $\mathcal{C}(r) = \{u_1, \dots, u_m\}$. The depth of $T$ is $1$ (root at depth $0$, children at depth $1$), so we consider any output sequence $o = (r, w, y) \in \Sigma^3$ where $w, y \in \Sigma$.

In the Decision Phase, the traversal order is $(u_1, \dots, u_m, r)$. There are two cases for how the algorithm $\mathscr{A}$ can generate $o$: 

\emph{Case 1: $w \in \mathcal{C}(r)$.} The algorithm $\mathscr{A}$ accepts the child node $w=u_i$ for some $1\le i\le m$ and then samples $y \sim \M_b(\cdot \mid r, w)$. The probability of accepting node $w$ is:
$$\tilde{p}_{u_i} \prod_{j< i}(1-\tilde{p}_{u_j}) = p_r(w),$$
where $p_r(w)$ is the marginal acceptance mass computed in the Allocation Phase. So the probability of $\mathscr{A}$ outputs $o$ equals
$$\Pr[\mathscr{A}=o \mid w \in \mathcal{C}(r)] = p_r(w) \M_b(y\mid r,w).$$
Then the contribution to the left-hand side of Equation~\eqref{eq:lossless-def} is
\begin{align*}
&\sum_{o' \subseteq o} \frac{\M_b(o) \Pr[\mathscr{A} = o'\mid T]}{\M_b(o')} \\ 
=\;&\Pr[\mathscr{A}=o \mid w \in \mathcal{C}(r)]\\
=\;& p_r(w) \M_b(y\mid r,w), \quad \text{(when $w\in\mathcal{C}(r)$)}   
\end{align*}

\emph{Case 2: $w \notin \mathcal{C}(r)$.} The algorithm $\mathscr{A}$ rejects all children and accepts the root node $r$, then samples $w$ from the residual distribution $\M_b(\cdot \mid r)$. The probability of accepting the root (fallback) is:
$$\prod_{j=1}^m (1-\tilde{p}_{u_j}) = 1 - \sum_{u \in \mathcal{C}(r)} p_r(u) = \sum_{u \notin \mathcal{C}(r)} p_r(u).$$
So the probability of $\mathscr{A}$ outputs $o'=(r,w)$ equals
\begin{align*}
&\Pr[\mathscr{A}=(r,w) \mid w \notin \mathcal{C}(r)] = \left(\sum_{u \notin \mathcal{C}(r)} p_r(u)\right) \frac{p_r(w)}{\sum_{u \notin \mathcal{C}(r)} p_r(u)} = p_r(w).
\end{align*}
Then the contribution to the left-hand side of Equation~\eqref{eq:lossless-def} is:
\begin{align*}
&\sum_{o' \subseteq o} \frac{\M_b(o) \Pr[\mathscr{A} = o'\mid T]}{\M_b(o')} \\ 
=\;& \M_b(y\mid r,w) \Pr[\mathscr{A} = (r,w) \mid w\notin\mathcal{C}(r)] \\
=\;& p_r(w) \M_b(y\mid r,w), \quad \text{(when $w\notin\mathcal{C}(r)$)}   
\end{align*}

Combining both cases, for any sequence $o=(r,w,y)$, we have:
\begin{align*}
&\sum_{T} p_{\rm draft}(T) \sum_{o' \subseteq o} \frac{\M_b(o)}{\M_b(o')} \Pr[\mathscr{A} = o'\mid T]\\
=\;& \sum_{T} p_{\rm draft}(T) \cdot p_r(w) \M_b(y\mid r,w) \\
=\;& \M_b(y\mid r,w)\mathbb{E}_{\mathcal{C}(r)} [p_r(w)].
\end{align*}
By the locally lossless property (Definition~\ref{def:loc-lossless}), for any specific token $w\in\Sigma$, we have $\mathbb{E}_{\mathcal{C}(r)}[p_r(w)] = \tilde{p}_r \M_b(w \mid r) = \M_b(w \mid r)$.
Thus, the base case $n=1$ holds.
    
\paragraph{Inductive step.}
Assume the theorem holds for all locally lossless trees with at most $K$ parent nodes. Consider a locally lossless tree $T$ with $K+1$ parent nodes. Let $v$ be a deepest parent node in $T$, which implies all children $\mathcal{C}(v)$ are leaves. Construct the reduced tree $T^{\rm sub} = T \setminus \mathcal{C}(v)$ by removing all children of $v$. Then $T^{\rm sub}$ contains exactly $K$ parent nodes and satisfies the induction hypothesis, that is for any sequence $o\in\Sigma^{L+1}$ (with $L$ be the length of the longest chain in the tree $T$),
\begin{equation}
\mathbb{E}_{T^{\rm sub}} \left[\sum_{o' \subseteq o} \frac{\M_b(o) \Pr[\mathscr{A} = o'\mid T^{\rm sub}]}{\M_b(o')}\right] = \M_b(o).
\label{eq:k-nodes-lossless}
\end{equation}

For any sequence $o'=(o_0, \dots, o_\ell)$, $|o'|\leq L+1$, as the output of $\mathscr{A}(T,\M_s,\M_b)$. Then $o_{\ell-1}$ is the acceptance node in Decision Phase.
We divide into three cases to prove the following equation
\begin{align}
\mathbb{E}_{\mathcal{C}(v)} \left[\sum_{o' \subseteq o} \frac{\M_b(o) \Pr[\mathscr{A} = o'\mid T]}{\M_b(o')}\right] = \sum_{o' \subseteq o} \frac{\M_b(o) \Pr[\mathscr{A} = o'\mid T^{\rm sub}]}{\M_b(o')}.\label{eq:exp-equal}
\end{align}

\emph{Case 1: the acceptance node $o_{\ell-1}$ is a preceding node of $\mathcal{C}(v)$.} In this case, we obviously have 
$$\Pr[\mathscr{A} = o'\mid T]=\Pr[\mathscr{A} = o'\mid T^{\rm sub}]. \quad \forall o'$$
Then \eqref{eq:exp-equal} always holds for this case.

\emph{Case 2: the acceptance node $o_{\ell-1}\in\mathcal{C}(v) \cup \{v\}$.}
Since Case 1 occurs with the same probability for both the trees $T$ and $T^{\rm sub}$, we only need to consider the conditional probability under the condition that Case 1 does not occur.

If $o_{\ell-1}\in\mathcal{C}(v)$, then the output sequence $o'$ of $\mathscr{A}$ is the long chain in tree $T$ (since the node $v$ is the deepest parent node of $T$). Thus, $\ell=L$ and the length $|o'|=L+1$. Note that $|o|=L+1$, then for all sequences $o\supseteq o'$ (i.e., $o=o'$), the probability \begin{align*}
&\sum_{o'\subseteq o}\frac{\M_b(o) \Pr[\mathscr{A} = o'\mid T]}{\M_b(o')} = p_v(o_{L-1}) \M_b(o_L\mid o_{[:L-1]}), \; (o_{L-1}\in\mathcal{C}(v)).
\end{align*}
If $o_{\ell-1}=v$ (i.e., $o'=o_{[:L-1]}$ and $\ell=L-1$),
then the probability \begin{align*}
&\sum_{o'\subseteq o}\frac{\M_b(o) \Pr[\mathscr{A} = o'\mid T]}{\M_b(o')} \\
=\;& \Pr[\mathscr{A} = o'\mid T]{\M_b(o_L\mid o')} \\
=\;& p_v(o_{L-1}){\M_b(o_L\mid o_{[:L-1]})},\; (o_{L-1}\notin\mathcal{C}(v)).
\end{align*}
To sum up, by the locally lossless of $T$, we know the Equation~\eqref{eq:exp-equal}
\begin{align*}
&\mathbb{E}_{\mathcal{C}(v)} \left[\sum_{o' \subseteq o} \frac{\M_b(o) \Pr[\mathscr{A} = o'\mid T]}{\M_b(o')}\right] \notag\\
=\;& \M_b(o_L\mid o_{[:L-1]}) \cdot \mathbb{E}_{\mathcal{C}(v)}[p_v(o_{L-1})] \notag\\
\overset{\eqref{eq:loc-lossless}}{=}\;& \M_b(o_L\mid o_{[:L-1]}) \cdot \tilde{p}_v \M_b(o_{L-1}\mid v)\\
=\;& \M_b(o_L\mid o_{[:L-1]}) \Pr[\mathscr{A} = (v,o_{L-1})\mid T^{\rm sub}] \\
=\;& \sum_{o' \subseteq o} \frac{\M_b(o) \Pr[\mathscr{A} = o'\mid T^{\rm sub}]}{\M_b(o')}
\end{align*}
holds for this case.

\emph{Case 3: the acceptance node $o_{\ell-1}$ is an after-node of $v$.} Since after rejection at node $v$, the remaining parts of $T$ and $T^{\rm sub}$ are identical under the Decision Phase. Therefore, we only need to prove that the expectation of Case 3 occurring for $T$ equals to the probability for $T^{\rm sub}$ rejecting $v$, i.e., to show that
\begin{align*}
\mathbb{E}_{\mathcal{C}(v)}\Pr\left( \mathscr{A} \text{ rejects } v \mid T \right) = \Pr\left( \mathscr{A} \text{ rejects } v \mid T^{\rm sub} \right).
\end{align*}

Observe that in the Decision Phase of $T^{\rm sub}$, when the algorithm reaches node $v$ (which is a leaf in $T^{\rm sub}$), it accepts $v$ with probability $\tilde{p}_v$ and rejects it with probability 
$$\Pr\left( \mathscr{A} \text{ rejects } v \mid T^{\rm sub}\right)=1-\tilde{p}_v.$$ 
In contrast, for the original tree $T$, when the algorithm traverses the local subtree $T_v = \{v\} \cup \mathcal{C}(v)$, the probability of rejecting all nodes in $T_v$ is:
\begin{align*}
&\mathbb{E}_{\mathcal{C}(v)}\Pr\left( \mathscr{A} \text{ rejects } v \mid T \right) \\
=\;&\mathbb{E}_{\mathcal{C}(v)} \left[\prod_{u \in \mathcal{C}(v)} (1-\tilde{p}_u) \cdot (1-\tilde{p}_v^{\rm res})\right]\\
=\;& \mathbb{E}_{\mathcal{C}(v)}\left[\frac{\left(1 - \sum_{u \in \mathcal{C}(v)} p_v(u)\right)\cdot p_v(\neg v)}{1 - \sum_{u \in \mathcal{C}(v)} p_v(u)}\right]\\
=\;& \mathbb{E}_{\mathcal{C}(v)}\left[ p_v(\neg v) \right].
\end{align*}
To avoid sacrificing the generality of this proof by using specific expressions (e.g., \eqref{eq:greedy-sample}--\eqref{eq:reject}), we solely rely on the identity $p_v(\neg v) = 1 - \sum_{u\in\Sigma}p_v(u)$ here. Together with the locally lossless property, we have
\begin{align*}
&\mathbb{E}_{\mathcal{C}(v)}\left[ p_v(\neg v) \right] = 1 -\sum_{u\in\Sigma} \mathbb{E}_{\mathcal{C}(v)}\left[p_v(u)\right] = 1 - \tilde{p}_v \sum_{u\in\Sigma}\M_b(u\mid v)=1 - \tilde{p}_v.
\end{align*}

Therefore, by the case analysis above, Equation~\eqref{eq:exp-equal} holds. Combining this with the observation that $\mathbb{E}_{T}[\cdot] = \mathbb{E}_{T^{\rm sub}}\left[\mathbb{E}_{\mathcal{C}(v)}[\cdot]\right]$, we can establish:
\begin{align*}
&\mathbb{E}_{T} \left[\sum_{o' \subseteq o} \frac{\M_b(o) \Pr[\mathscr{A} = o'\mid T]}{\M_b(o')}\right] \\
=\;&\mathbb{E}_{T^{\rm sub}} \mathbb{E}_{\mathcal{C}(v)}\left[\sum_{o' \subseteq o} \frac{\M_b(o) \Pr[\mathscr{A} = o'\mid T]}{\M_b(o')}\right] \\
\overset{\eqref{eq:exp-equal}}{=}\;& \mathbb{E}_{T^{\rm sub}} \left[\sum_{o' \subseteq o} \frac{\M_b(o) \Pr[\mathscr{A} = o'\mid T^{\rm sub}]}{\M_b(o')}\right] 
\overset{\eqref{eq:k-nodes-lossless}}{=} \M_b(o).
\end{align*}
\end{proof}

\subsection{Conditional Optimality of UniVer}

\begin{proof}[Proof of Theorem \ref{thm:optimality}]
Given any non-leaf node $v \in T$ with fixed $\tilde{p}_v$ and children $\mathcal{C}(v)=\{u_1,\dots,u_m\}$ generated by Greedy sampling (i.e., $H_v = \{u_1,\dots,u_{m-1}\}$ deterministic and $u_m \sim \M_s^\neg(\cdot\mid v)$), the conditional acceptance rate for $\mathcal{C}(v)$ is
\begin{align*}
&\alpha^*_{\rm UniVer} (\tilde{p}_v\M_b(\cdot \mid v), \M_s(\cdot \mid v)) \\
=\;& \E_{\mathcal{C}(v)}\left[\sum_{j=1}^m p_v(u_j) \;\middle|\; \tilde{p}_v\right] 
= \sum_{j=1}^m \E_{\mathcal{C}(v)}\left[ p_v(u_j) \;\middle|\; \tilde{p}_v\right].
\end{align*}

For the deterministic tokens $t \in H_v$, by local losslessness (Definition~\ref{def:loc-lossless}),
\begin{equation}
\E_{\mathcal{C}(v)}\left[p_v(t) \;\middle|\; \tilde{p}_v\right] = \tilde{p}_v\M_b(t\mid v).
\end{equation}
For the sampled node $u_m$, via Eq.~\eqref{eq:greedy-sample},
{
\begin{align*}
\E_{\mathcal{C}(v)}\left[p_v(u_m) \;\middle|\; \tilde{p}_v\right]
&= \sum_{u_m\in\Sigma} \M_s^\neg(u_m\mid v)p_v(u_m) \\
&= \sum_{u \in \Sigma} \min\{\tilde{p}_v\M_b(u \mid v), \M_s^{\neg}(u \mid v)\}.
\end{align*}}%
Therefore, the conditional acceptance rate of UniVer equals Eq.~\eqref{eq:cond-optimal}.

To show optimality, note that under the constraint of local losslessness, the marginal $\E_{\mathcal{C}(v)}[p_v(t) \mid \tilde{p}_v] = \tilde{p}_v\M_b(t\mid v)$ is fixed for all $t \in \Sigma$. The only freedom is in allocating the mass for the sampled node $u_m$. To maximize $\E_{\mathcal{C}(v)}[p_v(u_m) \mid \tilde{p}_v]$, we solve the following scaled optimal transport problem conditioned on $\tilde{p}_v$:
\begin{align*}
&\max \sum_{u\in \Sigma} \M_s^\neg (u\mid v) p_v(u) \\
&{\rm s.t.} \;\E_{\mathcal{C}(v)}[p_v(t) \mid \tilde{p}_v] = \tilde{p}_v \M_b(t \mid v), \;\;\forall t\in\Sigma.
\end{align*}
Equivalently, letting $\gamma_v(u,t):=\M_s^\neg (u \mid v) p_v(t)$ for $u\in\Sigma, t\in\Sigma^+:=\Sigma\cup\{\neg v\}$, the Scaled-OT problem is 
\begin{align}
&\max \sum_{u\in\Sigma} \gamma_v(u,u) \label{eq:scaled-ot}\\
{\rm s.t.} 
&\; \sum_{u\in\Sigma}  \gamma_v(u,\neg v) = 1-\tilde{p}_v, \notag\\
&\;\sum_{u\in\Sigma}  \gamma_v(u,t) = \tilde{p}_v \M_b(t \mid v), \;\;\forall t\in\Sigma, \notag \\
&\;\sum_{t\in\Sigma^+}  \gamma_v(u,t) = \M_s^\neg (u \mid v), \;\;\forall u\in\Sigma, \notag\\
&\; \gamma_v(u,t)\geq 0,\;\;\forall (u,t)\in\Sigma\times\Sigma^+. \notag
\end{align}
For the scaled optimal transport problem \eqref{eq:scaled-ot}, the objective is to maximize the total mass on the diagonal $\sum_{u\in\Sigma} \gamma_v(u,u)$. By the marginal constraints, for each $u\in\Sigma$  we have:
\begin{equation}
\gamma_v(u,u) \leq \sum_{t\in\Sigma^+}\gamma_v(u,t) = \M_s^\neg(u|v),
\end{equation}
and similarly for the column marginal:
\begin{equation}
\gamma_v(u,u) \leq \sum_{u'\in\Sigma}\gamma_v(u',u) = \tilde{p}_v\M_b(u|v).
\end{equation}
Therefore, $$\gamma_v(u,u)\leq\min\{\M_s^\neg(u\mid v), \tilde{p}_v \M_b(u\mid v)\}$$
for each $u$, yielding the upper bound:
\begin{equation}
\sum_{u\in\Sigma}\gamma_v(u,u) \leq \sum_{u\in\Sigma}\min\{\M_s^\neg(u\mid v), \tilde{p}_v\M_b(u\mid v)\}.
\end{equation}
This bound is achieved by our UniVer, completing the proof of conditional optimality.
\end{proof}

\subsection{Superiority of UniVer over Greedy Method}

Throughout this subsection, $N_{\mathscr A}(T)$ denotes the number of
accepted draft tokens produced by verification algorithm $\mathscr A$ on
the draft tree $T$, excluding the root token and the final token sampled
from the target model. All expectations are taken over the randomness in
draft-tree generation.

We assume the standard subtree factorization induced by autoregressive
drafting: conditioned on a node prefix $v$, the randomness used to
generate the proper descendants of $v$ is independent of the randomness
outside the subtree $T_v$.

For clarity, we write $N_{\rm Mod}(T;\rho)$ for the acceptance length of
modified UniVer when the root of $T$ is assigned effective acceptance
probability $\rho\in[0,1]$. Standard UniVer corresponds to $\rho=1$.

\begin{lemma}[Scaled Superiority]
\label{lem:Scaled-superiority}
For any finite draft tree $T$ with root $r$ and any root effective
acceptance probability $\rho\in[0,1]$, we have
\begin{equation}
\label{eq:inductive-ineq}
\E_T\!\left[N_{\rm Mod}(T;\rho)\right]
\geq
\rho\,
\E_T\!\left[N_{\rm Greedy}(T)\right].
\end{equation}
\end{lemma}

\begin{proof}
We proceed by induction on the number of parent nodes $n$ in $T$.

\paragraph{Base case ($n=1$).}
The tree consists only of the root $r$ and its children
$\mathcal C(r)$. Therefore,
\begin{align*}
\E_T\!\left[N_{\rm Mod}(T;\rho)\right]
&=
\alpha^*_{\rm UniVer}
\bigl(
\rho\M_b(\cdot\mid r),
\M_s(\cdot\mid r)
\bigr)
\\
&\geq
\rho\,
\alpha^*_{\rm Greedy}
\bigl(
\M_b(\cdot\mid r),
\M_s(\cdot\mid r)
\bigr)
\\
&=
\rho\,
\E_T\!\left[N_{\rm Greedy}(T)\right].
\end{align*}
The inequality follows directly from
Eq.~\eqref{eq:greedy-acc} and Eq.~\eqref{eq:cond-optimal}. In particular,
for every $a,b\geq 0$ and $\rho\in[0,1]$,
\[
\min\{\rho a,b\}
\geq
\rho\min\{a,b\}.
\]

\paragraph{Inductive step.}
Assume that Lemma~\ref{lem:Scaled-superiority} holds for all draft trees
with at most $k$ parent nodes. Consider a tree $T$ with $k+1$ parent
nodes, where $k\geq 1$.

Let
\[
\mathcal C(r)=\{v_1,\ldots,v_m\}
\]
be the ordered children of the root, where $v_m$ is the sampled child.
Since $T$ has at least two parent nodes, at least one child of $r$ is a
non-leaf node. Let $v_j$ be the first non-leaf child in this order.
Thus, $v_1,\ldots,v_{j-1}$ are leaves.

Define:
\begin{itemize}
    \item $T_{v_j}$ as the subtree rooted at $v_j$, including $v_j$ and
    all its descendants;
    \item $T^{\rm new}$ as the tree obtained from $T$ by removing all
    proper descendants of $v_j$, so that $v_j$ becomes a leaf.
\end{itemize}

Both $T_{v_j}$ and $T^{\rm new}$ contain at most $k$ parent nodes.

Let
\begin{equation}
\label{eq:reach-vj}
H_j
:=
\prod_{\ell=1}^{j-1}
\left(1-\tilde p_{v_\ell}\right).
\end{equation}
Since $v_1,\ldots,v_{j-1}$ are leaves, $H_j$ is the conditional
probability that the Decision Phase reaches the subtree rooted at
$v_j$.

For any prefix token $v$ and any $\eta\in[0,1]$, define
\begin{align}
U(v,\eta)
&:=
\E_{T_v}
\!\left[
N_{\rm Mod}(T_v;\eta)
\mid v
\right],
\label{eq:def-U}
\\
G(v)
&:=
\E_{T_v}
\!\left[
N_{\rm Greedy}(T_v)
\mid v
\right].
\label{eq:def-G}
\end{align}
Because the root token of a subtree is not counted in its acceptance
length, replacing the leaf $v_j$ in $T^{\rm new}$ by the full subtree
$T_{v_j}$ adds exactly the accepted descendants of $v_j$. Therefore,
\begin{equation}
\label{eq:mod-recursion}
\begin{split}
\E_T\!\left[N_{\rm Mod}(T;\rho)\right]
={}&
\E_{T^{\rm new}}
\!\left[
N_{\rm Mod}(T^{\rm new};\rho)
\right]
\\
&+
\E_{T^{\rm new}}
\!\left[
H_j\,
U\!\left(v_j,\tilde p_{v_j}\right)
\right].
\end{split}
\end{equation}

The value $\tilde p_{v_j}$ may depend on the root-level allocation
randomness contained in $T^{\rm new}$. However, after conditioning on
$T^{\rm new}$, it is fixed. Since the induction hypothesis holds for
every root effective probability in $[0,1]$, it can be applied
pointwise to $T_{v_j}$:
\begin{equation}
\label{eq:conditional-subtree-bound}
U\!\left(v_j,\tilde p_{v_j}\right)
\geq
\tilde p_{v_j}\,G(v_j).
\end{equation}

Applying the induction hypothesis also to $T^{\rm new}$ and substituting
Eq.~\eqref{eq:conditional-subtree-bound} into
Eq.~\eqref{eq:mod-recursion}, we obtain
\begin{align}
\E_T\!\left[N_{\rm Mod}(T;\rho)\right]
\geq{}&
\rho\,
\E_{T^{\rm new}}
\!\left[
N_{\rm Greedy}(T^{\rm new})
\right]
\nonumber\\
&+
\E_{T^{\rm new}}
\!\left[
H_j\tilde p_{v_j}G(v_j)
\right].
\label{eq:mod-intermediate}
\end{align}

By construction of the conditional hazards in the Decision Phase,
\begin{equation}
\label{eq:marginal-vj}
p_r^{\rm Mod}(v_j)
=
H_j\tilde p_{v_j},
\end{equation}
where $p_r^{\rm Mod}(v_j)$ denotes the marginal probability mass
allocated by modified UniVer to the root child $v_j$. Hence
Eq.~\eqref{eq:mod-intermediate} becomes
\begin{equation}
\label{eq:mod-intermediate-2}
\begin{split}
\E_T\!\left[N_{\rm Mod}(T;\rho)\right]
\geq{}&
\rho\,
\E_{T^{\rm new}}
\!\left[
N_{\rm Greedy}(T^{\rm new})
\right]
\\
&+
\E_{T^{\rm new}}
\!\left[
p_r^{\rm Mod}(v_j)G(v_j)
\right].
\end{split}
\end{equation}

The corresponding decomposition for Greedy is
\begin{equation}
\label{eq:greedy-recursion}
\begin{split}
\E_T\!\left[N_{\rm Greedy}(T)\right]
={}&
\E_{T^{\rm new}}
\!\left[
N_{\rm Greedy}(T^{\rm new})
\right]
\\
&+
\E_{T^{\rm new}}
\!\left[
p_r^{\rm Greedy}(v_j)G(v_j)
\right],
\end{split}
\end{equation}
where $p_r^{\rm Greedy}(v_j)$ is the marginal root-level acceptance
probability of $v_j$ under Greedy.

It remains to prove
\begin{equation}
\label{eq:weighted-local-comparison}
\E_{T^{\rm new}}
\!\left[
p_r^{\rm Mod}(v_j)G(v_j)
\right]
\geq
\rho\,
\E_{T^{\rm new}}
\!\left[
p_r^{\rm Greedy}(v_j)G(v_j)
\right].
\end{equation}

We distinguish two cases.

\paragraph{Case 1: $v_j$ is a deterministic child.}
Suppose
\[
v_j\in\{v_1,\ldots,v_{m-1}\}.
\]
The token value of $v_j$ is fixed. By the subtree-factorization
assumption, conditioned on $v_j$, the randomness used to generate
$T_{v_j}$ is independent of the root-level allocation randomness.
Consequently, $G(v_j)$ is constant with respect to the expectation over
the root-level randomness in $T^{\rm new}$.

By local losslessness of modified UniVer,
\begin{equation}
\label{eq:mod-local-marginal}
\E_{T^{\rm new}}
\!\left[
p_r^{\rm Mod}(v_j)
\right]
=
\rho\,\M_b(v_j\mid r).
\end{equation}
Similarly, the single-layer Greedy verifier is lossless and therefore
\begin{equation}
\label{eq:greedy-local-marginal}
\E_{T^{\rm new}}
\!\left[
p_r^{\rm Greedy}(v_j)
\right]
=
\M_b(v_j\mid r).
\end{equation}
It follows that
\begin{align*}
\E_{T^{\rm new}}
\!\left[
p_r^{\rm Mod}(v_j)G(v_j)
\right]
&=
G(v_j)\,
\E_{T^{\rm new}}
\!\left[
p_r^{\rm Mod}(v_j)
\right]
\\
&=
G(v_j) \cdot \rho\,\M_b(v_j\mid r)
\\
&=
\rho\,
G(v_j)\,
\E_{T^{\rm new}}
\!\left[
p_r^{\rm Greedy}(v_j)
\right]
\\
&=
\rho\,
\E_{T^{\rm new}}
\!\left[
p_r^{\rm Greedy}(v_j)G(v_j)
\right].
\end{align*}

\paragraph{Case 2: $v_j=v_m$ is the sampled child.}
Let
\[
q(u):=\M_s^\neg(u\mid r).
\]
For every $u$ in the support of $q$, modified UniVer assigns
\begin{equation}
\label{eq:sampled-mod-prob}
p_r^{\rm Mod}(u)
=
\min\left\{
1,
\frac{\rho\,\M_b(u\mid r)}
     {q(u)}
\right\},
\end{equation}
whereas Greedy assigns
\begin{equation}
\label{eq:sampled-greedy-prob}
p_r^{\rm Greedy}(u)
=
\min\left\{
1,
\frac{\M_b(u\mid r)}
     {q(u)}
\right\}.
\end{equation}
Since $\rho\in[0,1]$,
\begin{equation}
\label{eq:pointwise-sampled-comparison}
p_r^{\rm Mod}(u)
\geq
\rho\,p_r^{\rm Greedy}(u)
\qquad
\text{for every }u\in\operatorname{supp}(q).
\end{equation}
Moreover, $G(u)\geq 0$. Therefore,
\begin{align*}
\E_{T^{\rm new}}
\!\left[
p_r^{\rm Mod}(v_m)G(v_m)
\right]
&=
\sum_{u\in\operatorname{supp}(q)}
q(u)\,
p_r^{\rm Mod}(u)\,
G(u)
\\
&\geq
\rho
\sum_{u\in\operatorname{supp}(q)}
q(u)\,
p_r^{\rm Greedy}(u)\,
G(u)
\\
&=
\rho\,
\E_{T^{\rm new}}
\!\left[
p_r^{\rm Greedy}(v_m)G(v_m)
\right].
\end{align*}

Thus Eq.~\eqref{eq:weighted-local-comparison} holds in both cases.
Combining Eq.~\eqref{eq:mod-intermediate-2} with
Eq.~\eqref{eq:greedy-recursion}, we obtain
\begin{align*}
\E_T\!\left[N_{\rm Mod}(T;\rho)\right]
&\geq
\rho\,
\E_{T^{\rm new}}
\!\left[
N_{\rm Greedy}(T^{\rm new})
\right]
\\
&\quad+
\rho\,
\E_{T^{\rm new}}
\!\left[
p_r^{\rm Greedy}(v_j)G(v_j)
\right]
\\
&=
\rho\,
\E_T\!\left[
N_{\rm Greedy}(T)
\right].
\end{align*}
This completes the induction.
\end{proof}

\begin{proof}[Proof of Theorem~\ref{thm:univer_vs_greedy}]
Standard UniVer corresponds to root effective acceptance probability
$\rho=1$. Therefore, Lemma~\ref{lem:Scaled-superiority} gives
\[
\E_T\!\left[N_{\rm UniVer}(T)\right]
=
\E_T\!\left[N_{\rm Mod}(T;1)\right]
\geq
\E_T\!\left[N_{\rm Greedy}(T)\right].
\]
Hence UniVer is never worse than Greedy in expected acceptance length.
\end{proof}

%%%%%%%%%%%%%%%%%%%%%%%%%%%%%%%%%%%%%%%%%%%%%%%%%%%%%%%%%%%%

\end{document}